\documentclass{article}

\usepackage{microtype}
\usepackage{graphicx}
\usepackage{float}
\usepackage{subfig}
\usepackage{booktabs} 
\usepackage{amssymb,amsmath,amsthm}
\usepackage{color}
\usepackage{epsfig}
\usepackage{url}
\usepackage{hyperref}
\usepackage{todonotes}
\usepackage{multirow}
\usepackage{array}
\usepackage{enumitem}

\hypersetup{
  colorlinks=true,
  citecolor=blue,
  linkcolor=red,
  urlcolor=blue,
  }

\newtheorem{theorem}{Theorem}

\newtheorem{definition}{Definition}
\newtheorem{corollary}{Corollary}

\renewcommand{\[}{\begin{eqnarray}}
\renewcommand{\]}{\end{eqnarray}}

\newcommand{\E}{\mathbb{E}}
\newcommand{\R}{\mathbb{R}}

\usepackage[accepted]{icml2020}

\icmltitlerunning{On the generalization of bayesian deep nets}

\begin{document}

\twocolumn[
\icmltitle{On the generalization of bayesian deep nets for multi-class classification}




\begin{icmlauthorlist}
\icmlauthor{Yossi Adi}{bar}
\icmlauthor{Yaniv Nemcovsky}{te}
\icmlauthor{Alex Schwing}{il}
\icmlauthor{Tamir Hazan}{te}
\end{icmlauthorlist}

\icmlaffiliation{bar}{Bar-Ilan University}
\icmlaffiliation{il}{University of Illinois}
\icmlaffiliation{te}{Technion}

\icmlcorrespondingauthor{Yossi Adi}{adios@lucillecrew.com}
\icmlcorrespondingauthor{Tamir Hazan}{tamir.hazan@technion.ac.il}

\icmlkeywords{Machine Learning, ICML}

\vskip 0.3in
]



\printAffiliationsAndNotice{}  


\begin{abstract}
Generalization bounds which assess the difference between the true risk and the empirical risk have been studied extensively. However, to obtain bounds, current techniques use strict assumptions such as a uniformly bounded or a Lipschitz loss function. To avoid these assumptions, in this paper, we propose a new generalization bound for Bayesian deep nets by exploiting the contractivity of the Log-Sobolev inequalities. Using these inequalities adds an additional loss-gradient norm term to the generalization bound, which is intuitively a surrogate of the model complexity. Empirically, we analyze the affect of this loss-gradient norm term using different deep nets.
\end{abstract}


\section{Introduction}
\label{sec:intro}

\begin{figure}[!t]
  \centering
  \subfloat{\includegraphics[width=0.5\textwidth, height=4cm, keepaspectratio]{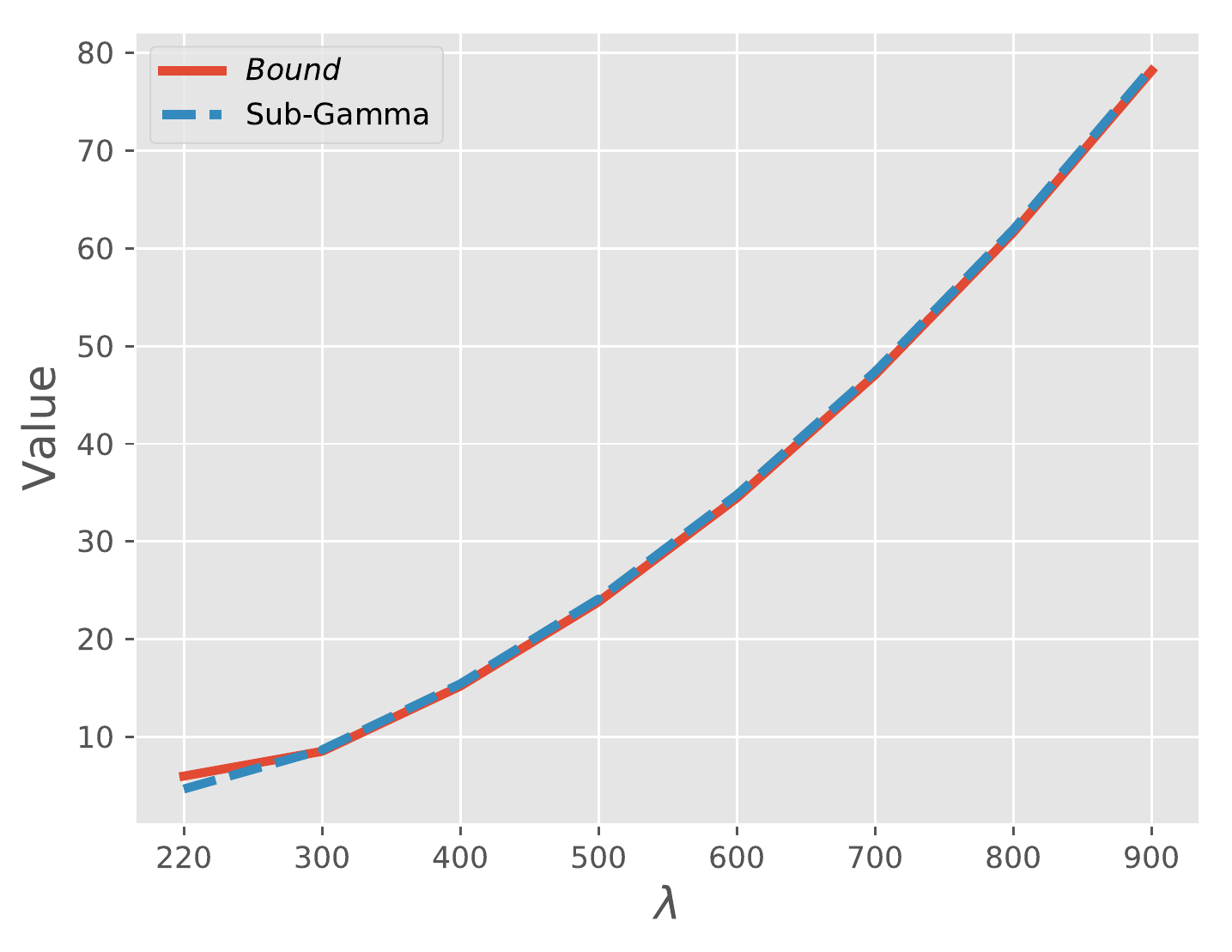}}
  \hfill
  \subfloat{\includegraphics[width=0.5\textwidth, height=4cm, keepaspectratio]{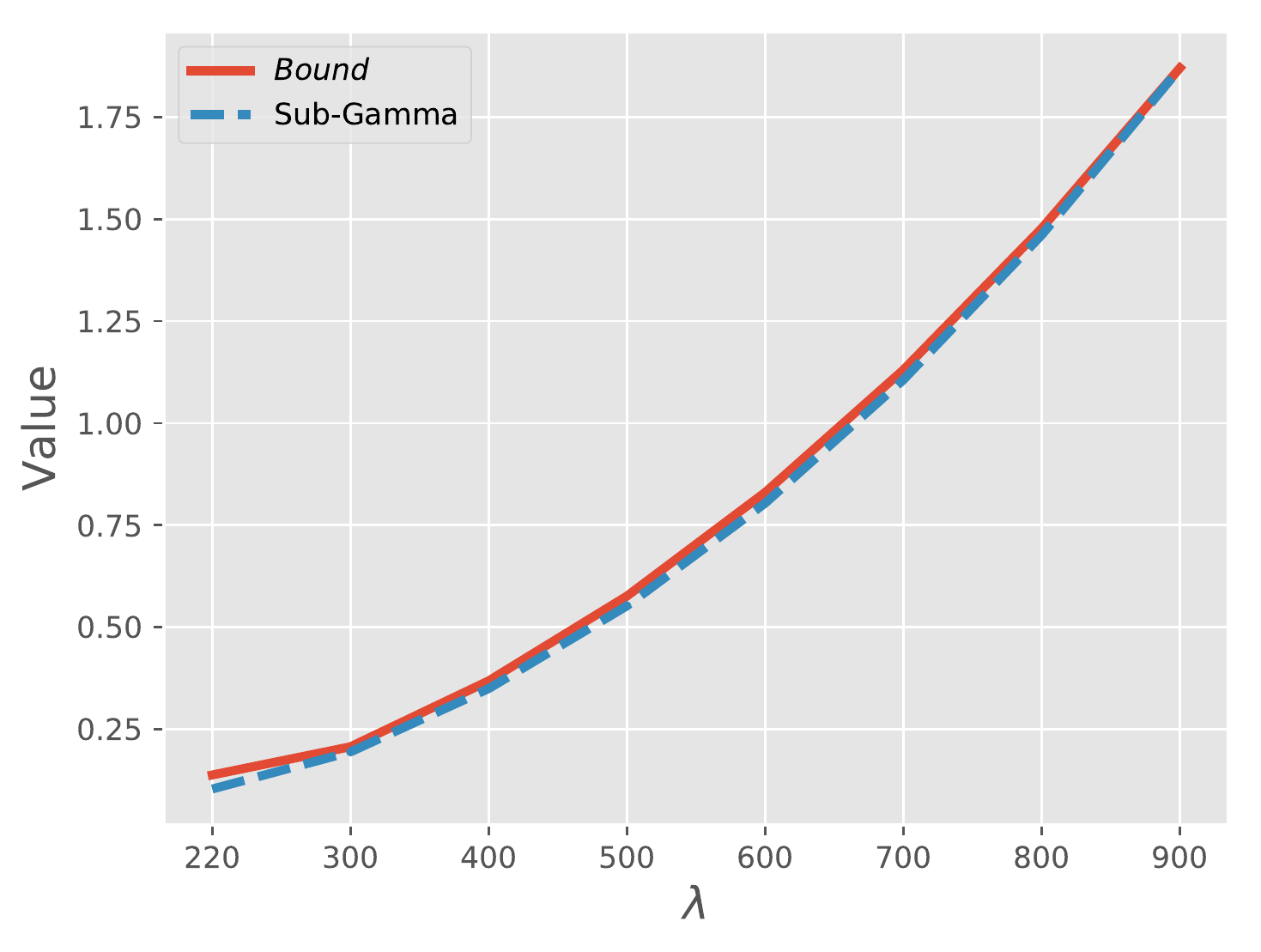}}
  \caption{The proposed bound as a function of $\lambda$ for both ResNet (top) and a Linear model (bottom). Notice, this suggests that the random variables $L_D(w) - L_S(w)$ for both ResNet and Linear models are sub-gamma (see Definition \ref{def:sg} and Theorem \ref{thm:alquir}). We obtain the results for ResNet using CIFAR-10 dataset and for the Linear model using MNIST dataset. The parameter for the sub-gamma fit are $v=1.0$ for both models and $c=1\mathrm{e}{-5}$ for ResNet and $c=1\mathrm{e}{-4}$ for the Linear model.}
  \label{fig:teaser}
\end{figure}

Deep neural networks are ubiquitous across  disciplines and often achieve state of the art results. Albeit deep nets  are able to encode highly complex input-output relations, in practice, they do not tend to overfit~\citep{Zhang16}. This tendency to not overfit has been investigated in numerous works on generalization bounds. Indeed, many generalization bounds apply to composite functions specified by deep nets. However, most of these bounds assume that the loss function is bounded or Lipschitz. Unfortunately, this assumption excludes plenty of deep nets and Bayesian deep nets that rely on the popular negative log-likelihood (NLL) loss. 

In this work we introduce a new PAC-Bayesian generalization bound for unbounded loss functions with unbounded gradient-norm, i.e., non-Lipschitz functions. This setting is closer to present-day deep net training, which uses the unbounded NLL loss and requires to avoid large gradient values during training so as to prevent exploding gradients. To prove the bound we utilize the contractivity of the log-Sobolev inequality~\citep{ledoux1999concentration}. It enables to bound the moment-generating function of the model risk. Our PAC-Bayesian bound adds a novel complexity term to existing PAC-Bayesian bounds: the expected norm of the loss function gradients computed with respect to the input. Intuitively this norm measures the complexity of the loss function, i.e., the model. In our work we prove that this complexity term is sub-gamma when considering linear models with the NLL loss, or more generally, for any linear model with a Lipschitz loss function. We also derive a bound for any Bayesian deep net, which permits to verify empirically that this complexity term is sub-gamma. See Figure \ref{fig:teaser}. 

This new term, which measures the complexity of the model, augments existing PAC-Bayesian bounds for bounded or Lipschitz loss functions which typically consist of two terms: (1) the empirical risk, which measures the fitting of the posterior over the parameters to the training data, and (2) the KL-divergence between the prior and the posterior distributions over the parameters, which measures the complexity of learning the posterior when starting from the prior over the parameters. 


\section{Related work}
\label{sec:related}

Generalization bounds for deep nets were explored in various settings. VC-theory provides both upper bounds and lower bounds to the network's VC-dimension, which are linear in the number of network parameters~\citep{Bartlett17b, Bartlett19}. While VC-theory asserts that such a model should overfit as it can learn any random labeling (e.g.,~\citet{Zhang16}), surprisingly, deep nets generally do not overfit. 

Rademacher complexity allows to apply data dependent bounds to deep nets~\citep{Bartlett02, Neyshabur15, Bartlett17, Golowich17, Neyshabur18}. These bounds rely on the loss and the Lipschitz constant of the network and consequently depend on a product of norms of weight matrices which scales exponentially in the network depth. \citet{Wei19} developed a bound that considers the gradient-norm over training examples. In contrast, our bound depends on average quantities of the gradient-norm and thus we answer an open question of \citet{Bartlett17} about the existence of bounds that depend on average loss and average gradient-norm, albeit in a PAC-Bayesian setting. PAC-Bayesian bounds that use Rademacher complexity have been studied by \citet{Kakade09, Yang19}. 

Stability bounds may be applied to unbounded loss functions and in particular to the negative log-likelihood (NLL) loss~\citep{Bousquet02, Rakhlin05, Shalev09, Hardt15, Zhang16}. However, stability bounds for convex loss functions, e.g., for logistic regression, do not apply to deep nets since they require the NLL loss to be a convex function of the parameters.  Alternatively,~\citet{Hardt15} and \citet{Kuzborskij17} estimate the stability of stochastic gradient descent dynamics, which strongly relies on early stopping. This approach results in weaker bounds for the non-convex setting. Stability PAC-Bayesian bounds for bounded and Lipschitz loss functions were developed by \citet{London17}. 

PAC-Bayesian bounds were recently applied to deep nets~\citep{McAllester13, Dziugaite17, Neyshabur17}. In contrast to our work, those related works all consider bounded loss functions. An excellent survey on PAC-Bayesian bounds was provided by~\citet{Germain16}.  \citet{Alquier16} introduced PAC-Bayesian bounds for linear classifiers with the hinge-loss by explicitly bounding its moment generating function. \citet{Alquier12} provide an analysis for PAC-Bayesian bounds with Lipschitz functions. Our work differs as we derive PAC-Bayesian bounds for non-Lipschitz functions. Work by \citet{Germain16} is closer to our setting and considers PAC-Bayesian bounds for the quadratic loss function. In contrast, our work considers the multi-class setting, and non-linear models. PAC-Bayesian bounds for the NLL loss in the online setting were put forward by~\citet{Takimoto00, Banerjee06, Bartlett13, Grunwald17}. The online setting does not consider the whole sample space and therefore is simpler to analyze in the Bayesian setting.

PAC-Bayesian bounds for the NLL loss function are intimately related to learning Bayesian inference~\citep{Germain16}. Recently many works applied various posteriors in Bayesian deep nets. \citet{Gal15,gal2016uncertainty} introduce a Bayesian inference approximation using Monte Carlo (MC) dropout, which approximates a Gaussian posterior using Bernoulli dropout. \citet{Srivastava14, Kingma15} introduced Gaussian dropout which effectively creates a Gaussian posterior that couples between the mean and the variance of the learned parameters and explored the relevant log-uniform priors. \citet{Blundell15, Louizos16} suggest to take a full Bayesian perspective and learn separately the mean and the variance of each parameter.

\section{Background}

Generalization bounds provide statistical guarantees on  learning algorithms. They assess how the learned parameters $w$ of a model perform on test data given the model's result on the training data $S = \{(x_1,y_1),\ldots,(x_m,y_m)\}$, where $x_i$ is the data instance and $y_i$ is the corresponding label. The performance of the parametric model is measured by a loss function $\ell(w,x,y)$. The risk of this model is its average loss, when the data instance and its label are sampled from the true but unknown distribution $D$. We denote the risk by $L_D(w) = \E_{(x,y) \sim D} \ell(w,x,y)$. The empirical risk is the average training set loss  $L_S(w) = \frac{1}{m}\sum_{i=1}^m \ell(w,x_i,y_i)$. 

PAC-Bayesian theory bounds the expected risk $\E_{w \sim q} L_D(w)$ of a model  when its parameters are averaged over the learned posterior distribution $q$. The parameters of the posterior distribution are learned from the training data $S$. In our work we start from the following PAC-Bayesian bound:

\begin{theorem}[\citet{Alquier16}]
\label{thm:alquir}
Let $KL(q||p) = \int q(w) \log(q(w)/p(w)) dw$ be the KL-divergence between two probability density functions $p,q$. For any $\lambda > 0$, for any $\delta \in (0,1]$ and for any prior distribution $p$, with probability at least $1-\delta$ over the draw of the training set $S$, the following holds simultaneously for any posterior distribution $q$: $\E_{w \sim q} [L_D(w)] \le$
\[
 \E_{w \sim q}[L_S(w)] + \frac{C(\lambda,p) +  KL(q || p) + \log(1/\delta)}{\lambda},\label{eq:pb}
\]
where 
\[
\label{eq:alquier}
C(\lambda,p) = \log \E_{w \sim p, S \sim D^m}  [ e^{ \lambda (L_D(w) - L_S(w))} ].
\]
\end{theorem}
Unfortunately, the complexity term $C(\lambda,p)$ of this bound is impossible to compute for large values of $\lambda$, as we show in our experimental evaluation. To deal with this complexity term, \citet{Alquier16, Germain16, Boucheron13} consider the sub-Gaussian assumption, which amounts to the bound $C(\lambda,p) \le \frac{\lambda^2 v}{2}$ for any $\lambda \in \R$ and some variance factor $v$. This assumption is also referred to as the Hoeffding assumption, which is related to  Hoeffding's lemma that is usually applied in PAC-Bayesian bounds to loss function that are uniformly bounded by a constant, i.e., for any $w,x,y$ simultaneously. 

Unfortunately, many loss functions that are used in practice are unbounded. In particular, the NLL loss function is unbounded even in the multi-class setting when $y$ is discrete. For instance, consider a  fully connected deep net, where the input vector of the $k$-th layer is a function of the parameters of all previous layers, i.e., $x_{k}(W_0,\ldots,W_{k-1})$. The entries of $x_{k}$ are computed from the response of its preceding layer, i.e.,  $W_{k-1} x_{k-1}$, followed by a transfer function $\sigma(\cdot)$, i.e., $x_{k} = \sigma(W_{k-1} x_{k-1})$. Since the NLL is define as $-\log p(y|x,w) = -(W_k x_k)_y + \log (\sum_{\hat y} e^{(W_k x_k)_{\hat y}} )$, the NLL loss increases with $r$ and is unbounded when $r \rightarrow \infty$ if the rows in $W_k$ consist of the vector $r x_k$. In our experimental validation in Section~\ref{sec:exp} we show that the unboundedness of the NLL loss results in a complexity term $C(\lambda,p)$ that is not sub-Gaussian. 

This complexity term $C(\lambda,p)$ influences the value of $\lambda$, which controls the convergence rate of the bound, as it weighs the complexity terms $C(\lambda,p)$ and $KL(q||p)$ by $\frac{1}{\lambda}$. Therefore, a tight bound requires $\lambda$ to be as large as possible. However, since $\lambda$ influences $C(\lambda,p)$ exponentially, one needs to make sure that $C(\lambda,p) < \infty$. In such cases one may use sub-gamma random variables \cite{Alquier16, Germain16, Boucheron13}: 
\begin{definition}
\label{def:sg}
The random variable $L_D(w) - L_S(w)$ is called sub-gamma if the complexity term $C(\lambda,p)$ in Theorem \ref{thm:alquir} satisfies $C(\lambda,p) \le \frac{\lambda^2 v}{2(1-\lambda c)}$  for every $\lambda$ such that $0 < \lambda < 1/c$.  
\end{definition}

In Corollary~\ref{cor:logistic} we prove that the complexity term $C(\lambda,p)$ is sub-gamma when considering linear models with the NLL loss, or more generally, for any linear model with a Lipschitz loss function. In Corollary~\ref{cor:nonlinear} we derive a bound on $C(\lambda,p)$ for any Bayesian deep net, which permits to verify empirically that $C(\lambda,p)$ is sub-gamma. 
    
\section{PAC-Bayesian bounds for smooth loss functions}

Our main theorem below shows that for smooth loss functions, the complexity term $C(\lambda,p)$ is bounded by the expected gradient-norm of the loss function with respect to the data $x$. This property is appealing since the gradient-norm contracts the network's output, as evident in its extreme case by the vanishing gradients property. In our experimental evaluation we show how this contractivity depends on the depth of the network and the variance of the prior and its affect on the generalization of Bayesian deep nets.

\begin{theorem}
\label{thm:pac-bound-general-2}
Consider the setting of Theorem \ref{thm:alquir} and assume $(x,y) \sim D$, $x$ given $y$ follows the Gaussian distribution and  $\ell(w,x,y)$ is a smooth loss function (e.g., the negative log-likelihood loss). Let $M(\alpha) \triangleq \E_{(x,y) \sim D} e^{\alpha (-\ell(w,x,y))}$. Then $C(\lambda,p) \le $
\begin{equation}
\label{eq:paper_bound}
\hspace{-0.3cm}\log\E_{w \sim p} e^{2\lambda \E_{(x, y) \sim D}\big [ \| \nabla_{x} \ell(w, x, y) \| ^2 \int_0^{\frac{\lambda}{m}} \frac{ e^{-\alpha \ell(w, x, y)} }{M(\alpha)}  d \alpha \big]}.
\end{equation}
\end{theorem}
Our proof technique  to show the main Theorem uses the log-Sobolev inequality for Gaussian distributions \citep{ledoux1999concentration} as we illustrate next.

\begin{proof}
The proof consists of three steps. First we use the statistical independence of the training samples to decompose the moment generating function
\[
\label{eq:mgfi}
\E_{S \sim D^m}  [ e^{ \lambda (L_D(w) - L_S(w)) } ] = e^{\lambda L_D(w)} M\Big(\frac{\lambda}{m} \Big)^m. 
\]
Next we consider the moment generating function $M(\frac{\lambda}{m}) \triangleq \E_{(\hat{x}, \hat{y}) \sim D} [ e^{ \frac{\lambda}{m} (- \ell(w, \hat{x}, \hat{y}))  } ]$ in its log-space, i.e., by considering the cumulant generating function $K(\alpha) \triangleq \frac{1}{\alpha} \log M(\alpha)$ and obtain the following equality: 
\[
\label{eq:herbst}
M\Big(\frac{\lambda}{m}\Big) = e^{-\frac{\lambda}{m} L_D(w) + \frac{\lambda}{m} \int_0^{\frac{\lambda}{m}} \frac{ \alpha M'(\alpha) -  M(\alpha) \log M(\alpha)}{\alpha^2 M(\alpha)} d\alpha }.
\]
Finally we use the log-Sobolev inequality for Gaussian distributions, $\alpha M'(\alpha) -  M(\alpha) \log M(\alpha) \le $
\[
\label{eq:logsob}
 2 \cdot \E_{(x,y) \sim D} [ e^{-\alpha \ell(w,x,y)} \alpha^2 \| \nabla_x \ell(w,x,y) \| ^2 ] 
\]
and complete the proof through some algebraic manipulations. 

The first step of the proof results in Eq.~\eqref{eq:mgfi}. To derive it we use the independence of the training samples: $\E_{S \sim D^m}  [ e^{ \lambda (L_D(w) - L_S(w)) } ]=$  
\[
&&e^{ \lambda L_D(w)} \E_{S \sim D^m}  [ e^{ \lambda \frac{1}{m} \sum_{i=1}^m (- \ell(w,x_i,y_i)) } ] \\
&=& e^{ \lambda L_D(w)} \prod_{i=1}^m \E_{(x_i,y_i) \sim D}  [ e^{\frac{\lambda}{m} (- \ell(w,x_i,y_i))}  ] \\
&=& e^{ \lambda L_D(w)} \Big( \E_{(x,y) \sim D}  [ e^{\frac{\lambda}{m} (- \ell(w,x_i,y_i))}  ]  \Big)^m. \label{eq:mgf_f}
\] 
The first equality holds since $e^{ \lambda L_D(w)}$ is a constant that is independent of the expectation $S \sim D^m$. The last equality holds since $(x_i,y_i) \sim D$ are identically distributed. 

The second step of the proof results in Eq.~\eqref{eq:herbst}. It relies on the relation of the moment generating function and the cumulant generating function $M(\frac{\lambda}{m}) = e^{\frac{\lambda}{m}  K(\frac{\lambda}{m})}$. The fundamental theorem of calculus asserts $K(\frac{\lambda}{m}) - K(0) = \int_0^{\lambda/m} K'(\alpha) d \alpha$, while $K'(\alpha)$ refers to the derivative at $\alpha$. We then compute $K'(\alpha)$ and $K(0)$: 
\[
K'(\alpha) &=& \frac{ \alpha M'(\alpha) -  M(\alpha) \log M(\alpha)}{\alpha^2 M(\alpha)}, \label{eq:k'}\\
K(0) &=& \lim_{\alpha \rightarrow 0^+} \frac{\log M(\alpha)}{\alpha}  = -L_D(w). 
\]   
The second equality follows from l'Hopital rule: $\lim_{\alpha \rightarrow 0^+} \frac{\log M(\alpha)}{\alpha} \stackrel{\text{l'Hopital}}{=} \frac{M'(0) / M(0)}{1}$ and recalling that $M(0)=1$ and $ M'(0) = -L_D(w)$.

The third and final step of the proof begins with applying the log-Sobolev inequality in Eq.~\eqref{eq:logsob} for Gaussian distributions. Combining it with Eq.~\eqref{eq:herbst} leads to the inequality $M(\frac{\lambda}{m}) \le$ 
\[
e^{-\frac{\lambda}{m} L_D(w) + \frac{\lambda}{m} \int_0^{\frac{\lambda}{m}} \frac{ 2 \cdot \E_{(x,y) \sim D} [ e^{-\alpha \ell(w,x,y)} \alpha^2 \| \nabla_x \ell(w,x,y) \| ^2 ]}{\alpha^2 M(\alpha)} d\alpha }.
\]

It is insightful to see that $M(\frac{\lambda}{m})^m \le$
\[
\label{eq:insight}
e^{-\lambda L_D(w)} e^{\lambda \int_0^{\frac{\lambda}{m}} \frac{ 2 \cdot \E_{(x,y) \sim D} [ e^{-\alpha \ell(w,x,y)} \alpha^2 \| \nabla_x \ell(w,x,y) \| ^2 ]}{\alpha^2 M(\alpha)} d\alpha }. 
\]
This is compelling because the $e^{-\lambda L_D(w)}$ term in the above inequality  cancels  with the term $e^{\lambda L_D(w)}$ in Eq.~\eqref{eq:mgfi}. This shows theoretically why the complexity term mostly depends on the gradient-norm. This observation also concludes the proof after rearranging  terms. 
\end{proof}

The above theorem can be extended to settings for which $x$ is sampled from any log-concave distribution, e.g., the Laplace distribution, cf.~\citet{gentil2005logarithmic}. For readability we do not discuss this generalization here. 

Eq.~\eqref{eq:mgfi} hints at the infeasibility of computing the complexity term $C(\lambda,p)$ directly for large values of $\lambda$: since the loss function is non-negative, the value of $e^{\lambda L_D(w)}$ grows exponentially with $\lambda$, while $e^{ \frac{\lambda}{m} (- \ell(w, \hat{x}, \hat{y}))  }$ diminishes to zero exponentially fast. These opposing quantities make evaluation of $C(\lambda,p)$ numerically infeasible. In contrast, our bound makes all computations in  log-space, hence its computation is feasible for larger values of $\lambda$, up to their sub-gamma interval $0< \lambda < 1/c$, see Table \ref{tab:optimize}.

Notably, the general bound in Theorem~\ref{thm:pac-bound-general-2} is more theoretical than practical: to estimate it in practice one needs to avoid the integration over $\alpha$. 
However, it is an important intermediate step to derive a practical bound for linear models with Lipschitz loss function, and a general bound for any smooth loss function as we discuss in the next two sections respectively.


\subsection{Linear models}
\label{sec:linear}

In the following we consider smooth loss functions over linear models in the multi-class setting, where $x \in \R^d$ is the data instance, $y\in\{1,\dots,k\}$ are the possible labels and the loss function takes the form $\ell(w,x,y) \triangleq \hat \ell(Wx, y)$. We also assume that $\hat \ell(t, y)$ is a Lipschitz function, i.e., $\|\nabla_t \hat \ell(t, y) \| \le L$. Included in these assumptions are the popular NLL loss  $-\log p(y|x,w) = -(W x)_y + \log (\sum_{\hat y} e^{(W x)_{\hat y}} )$ that is used in logistic regression and the multi-class hinge loss $\max_{\hat y} \{ (W x)_{\hat y} - (W x)_y + 1[y \ne \hat y]\}$ that is used in support vector machines (SVMs).  

\begin{corollary}
\label{cor:logistic}
Consider smooth loss functions over linear models $\ell(w,x,y) \triangleq \hat \ell(Wx, y)$, with Lipschitz constant $L$, i.e., $\|\nabla_t \hat \ell(t, y)\| \le L$. Under the conditions of Theorem \ref{thm:pac-bound-general-2} with Gaussian prior distribution $p \sim N(0,\sigma_p)$ and variance $\sigma_p^2$ for which $\lambda \le \sqrt{\frac{m}{8}} / L \sigma_p$, we obtain $C(\lambda,p) \le kd \log(2)$.  
\end{corollary}
\begin{proof}
This bound is derived by applying Theorem \ref{thm:pac-bound-general-2}. We begin by realizing the gradient of $\hat \ell(Wx, y)$ with respect to $x$. Using the chain rule, $\nabla_x \hat \ell(Wx, y) = W^\top \nabla_{Wx} \hat \ell(Wx,y)$. Hence, we obtain for the gradient norm  $\|\nabla_x \hat \ell(Wx, y)\|^2 \le \|\hat \ell(Wx,y) \|^2 \cdot \sum_{y=1}^k \sum_{j=1}^d w_{y,j}^2 \le L^2 \sum_{y=1}^k \sum_{j=1}^d w_{y,j}^2 $. Plugging this result into Eq.~\eqref{eq:paper_bound} we obtain the following bound for its exponent: 
\[
&&\hspace{-1.2cm} \E_{(x, y) \sim D}\big [ \| \nabla_{x} \ell(w, x, y) \| ^2 \int_0^{\frac{\lambda}{m}} \frac{ e^{-\alpha \ell(w, x, y)} }{M(\alpha)}  d \alpha \big] \\
&&\hspace{-1cm}   \le L^2 \sum_{y=1}^k \sum_{j=1}^d w_{y,j}^2 \cdot \E_{(x, y) \sim D}\big [\int_0^{\frac{\lambda}{m}} \frac{ e^{-\alpha \ell(w, x, y)} }{M(\alpha)}  d \alpha \big] \\
&&\hspace{-1cm}   = L^2 \sum_{y=1}^k \sum_{j=1}^d w_{y,j}^2 \int_0^{\frac{\lambda}{m}} \frac{ \E_{(x, y) \sim D}\big [ e^{-\alpha \ell(w, x, y)} \big]}{M(\alpha)}  d \alpha. 
\]
Since $M(\alpha) \triangleq \E_{(x, y) \sim D}\big [ e^{-\alpha \ell(w, x, y)} \big]$, the ratio in the integral equals one and the integral $\int_0^{\frac{\lambda}{m}} d \alpha = \frac{\lambda}{m}$. Combining these results we obtain:
\[
C(\lambda,p) & \le & \log \Big(\E_{w \sim p} e^{\frac{2\lambda^2 L^2}{m} \sum_{y,j} w_{y,j}^2} \Big).
\]
Finally, whenever $\lambda L \sigma_p  \le \sqrt{m/8}$ we follow the Gaussian integral and derive the bound
\[
\E_{w \sim p} e^{\frac{2\lambda^2 L^2}{m} \sum_{y,j} w_{y,j}^2} \le \Big(\frac{m}{m-8 L^2 \lambda^2 \sigma_p^2}\Big)^{kd}.
\]

\end{proof}

The above corollary provides a PAC-Bayesian bound for classification using the NLL loss, and shows $C(\lambda,p)$ is sub-gamma in the interval $0 < \lambda < \sqrt{\frac{m}{8}} / L \sigma_p$. Interestingly, it can achieve a rate of $\lambda = m$, albeit with variance of the prior of $1/m$. In our experimental evaluation we show that it is better to achieve lower rate, i.e., $\lambda = \sqrt{m}$ while using a prior with a fixed variance, i.e., $\sigma_p \approx 0.1$. The above bound also extends the result of \citet{Alquier16} for binary hinge-loss to the multi-class hinge loss (cf. \citet{Alquier16}, Section 6). Unfortunately, the above bound cannot be applied to non-linear loss functions, since their gradient-norm is not bounded, as evident by the exploding gradients property in deep nets. 

\begin{figure}[t]
  \centering
	\includegraphics[width=0.5\textwidth,height=4.5cm, keepaspectratio]{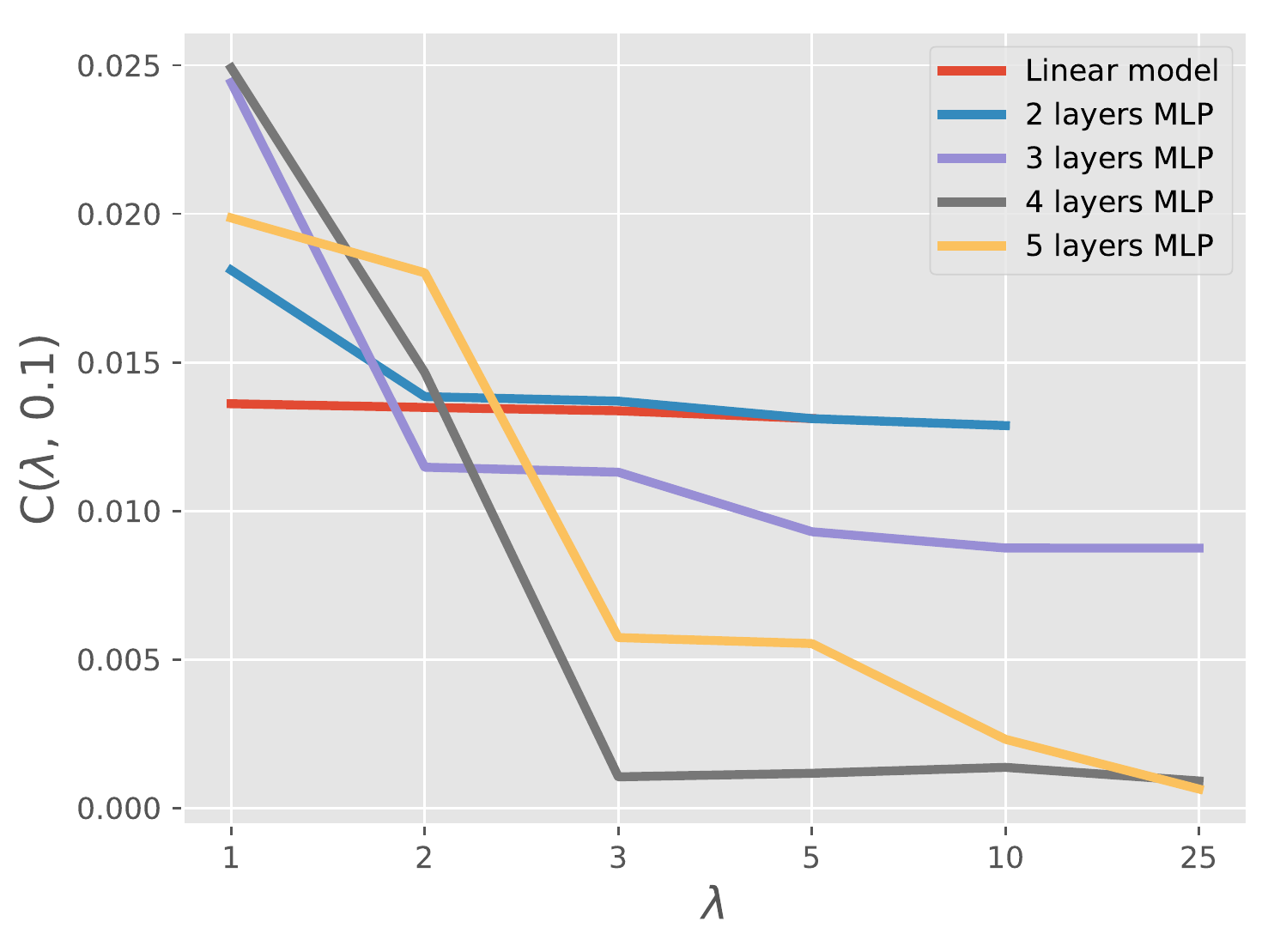}
	\caption{Estimating the complexity term $C(\lambda,p)$ in Eq.~\ref{eq:alquier} over MNIST, for a linear model and $4$ MLPs of depth $d=2,...,5$. We are able to compute the complexity term $C(\lambda,p)$ for $\lambda \le 25$ and even to smaller $\lambda$ in for the linear model ($\lambda \le 5$). Standard bounds for MNIST require $\lambda$ to be at least $240$, which is the square root of the training sample size. In all settings we use variance of 0.1 for the prior distribution over the model parameters.}
 	\label{fig:alquier}
\end{figure}

\subsection{Non-linear models}
In the following we derive a generalization bound for on-average bounded loss functions and on-average bounded gradient-norms. 
\begin{corollary}
\label{cor:nonlinear}
Consider smooth loss functions that are on-average bounded, i.e., $\E_{(x,y) \sim D} \ell(w,x,y) \le b$. Under the conditions of Theorem \ref{thm:pac-bound-general-2}, for any $0< \lambda \le m$ we obtain 
\[
\label{eq:mgf_bound}
C(\lambda,p) \le  \log\E_{w \sim p} e^{\frac{2 \lambda^2 e^{b} }{m} \E_{(x, y) \sim D}\big [ \| \nabla_{x} \ell(w, x, y) \| ^2\big]}.
\]
\end{corollary}
\begin{proof}
This bound is derived by applying Theorem \ref{thm:pac-bound-general-2} and bounding $\int_0^1 \frac{ e^{-\alpha \ell(w, x, y)} }{M(\alpha)}  d \alpha \le e^b $. We derive this bound in three steps: 
\begin{itemize}[noitemsep,topsep=0pt,parsep=0pt,partopsep=0pt]
\item From $\ell(w,x,y) \ge 0$ we obtain $e^{-\alpha \ell(x,x,y)} \le 1$. 
\item We lower bound $M(\alpha) \ge M(1)$ for any $0 \le \alpha \le \lambda/m$: First we note that $0 < \lambda \le m$, therefore we consider $0 \le \alpha \le 1$. Also, since $\ell(w,x,y) \ge 0$ the function $e^{-\alpha \ell(w,x,y)}$ is monotone in $\alpha$ within the unit interval, i.e.,  for $0 \le \alpha_1  \le \alpha_2 \le 1$ there holds $e^{-\alpha_1 \ell(w,x,y)} \ge e^{-\alpha_2 \ell(w,x,y)}$ and consequently $M(\alpha) \ge M(1)$ for any $\alpha \le 1$.  
\item The assumption $\E_{(x,y) \sim D} [-\ell(w,x,y)] \ge -b$ and the monotonicity of the exponential  function result in the lower bound $e^{\E_{(x,y) \sim D} [-\ell(w,x,y)] } \ge e^{-b}$. From  convexity of the exponential function, $M(1) = \E_{(x,y) \sim D} e^{-\ell(w,x,y) } \ge  e^{\E_{(x,y) \sim D} [-\ell(w,x,y)] }$ and the lower bound $M(1) \ge e^{-b}$ follows. 
\end{itemize}
Combining these bounds we derive the upper bound $\int_0^1 \frac{ e^{-\alpha \ell(w, x, y)} }{M(\alpha)}  d \alpha \le \int_0^1 \frac{ 1 }{e^{-b}}  d \alpha = e^b$, and the result follows. 
\end{proof}
The above derivation upper bounds the complexity term $C(\lambda,p)$ by the expected gradient-norm of the loss function, i.e., the flow of its gradients through the architecture of the model. It provides the means to empirically show that $C(\lambda,p)$ is sub-gamma (see Section~\ref{sec:subgamma}). In particular, we show empirically that the rate of the bound $\lambda$ can be as high as $m$, dependent on the gradient-norm. This is a favorable property, since the convergence of the bound scales as $1/\lambda$. Therefore, one would like to avoid exploding gradient-norms, as this effectively harms the true risk bound. While one may achieve a fast rate bound by forcing the gradient-norm to vanish rapidly, practical experience shows that vanishing gradients prevent the deep net from fitting the model to the training data when minimizing the empirical risk. In our experimental evaluation we demonstrate the influence of the expected gradient-norm on the bound of the true risk.

\section{Experiments}
\label{sec:exp}
\begin{figure*}[!t]
  \centering
  \subfloat{\includegraphics[width=0.3\textwidth, height=5cm, keepaspectratio]{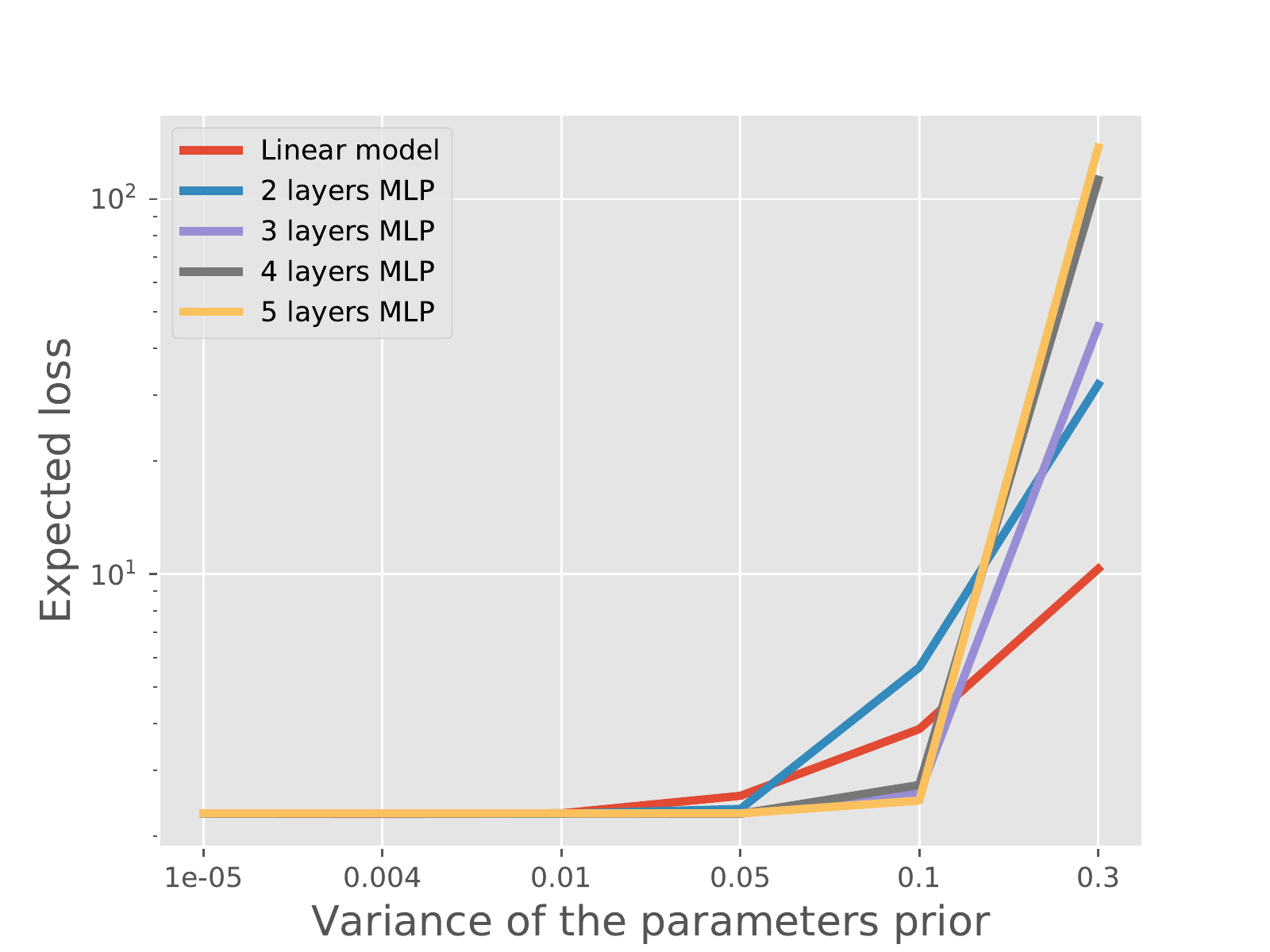}}
  \subfloat{\includegraphics[width=0.3\textwidth, height=5cm, keepaspectratio]{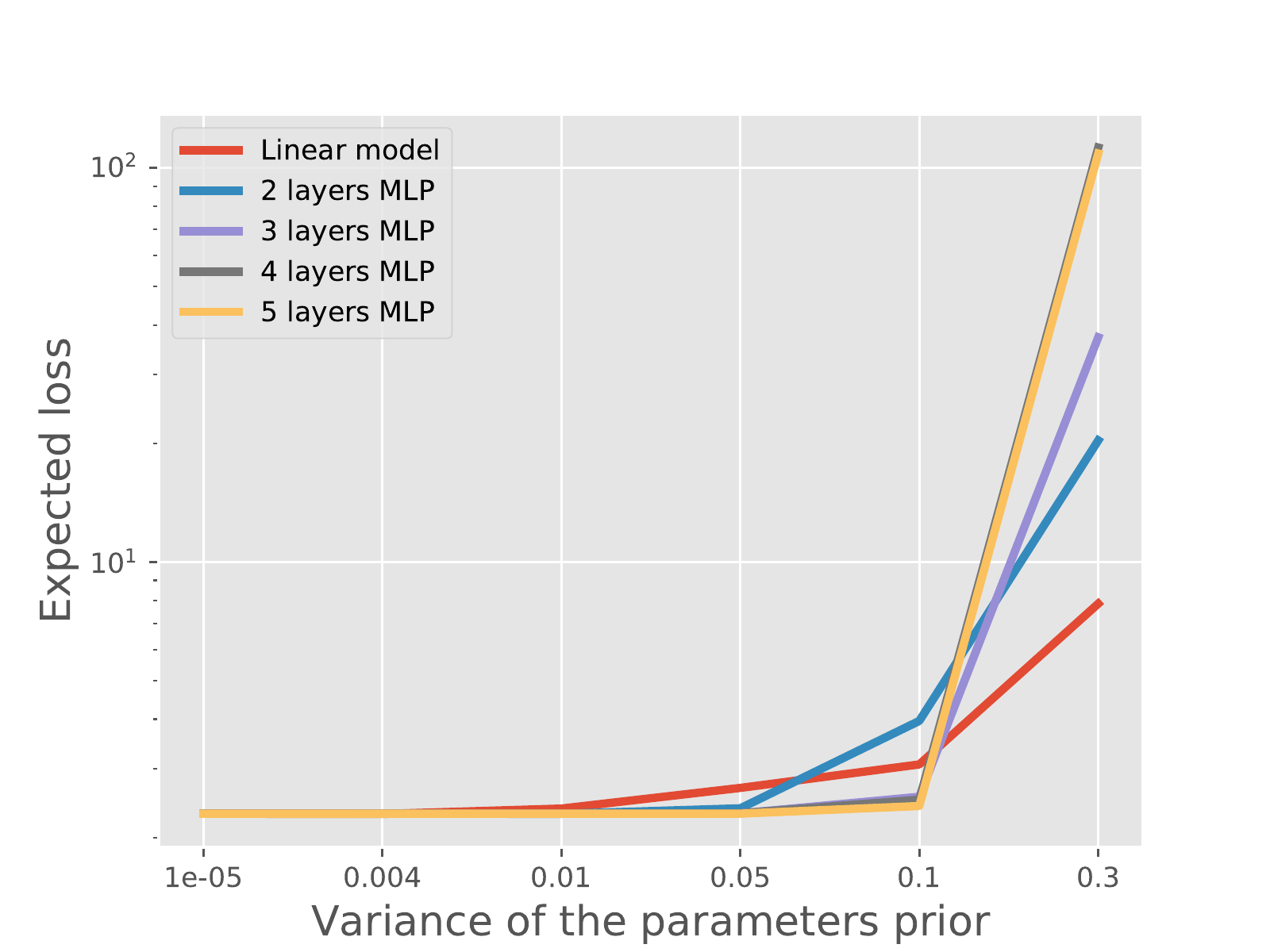}}
  \subfloat{\includegraphics[width=0.3\textwidth, height=5cm, keepaspectratio]{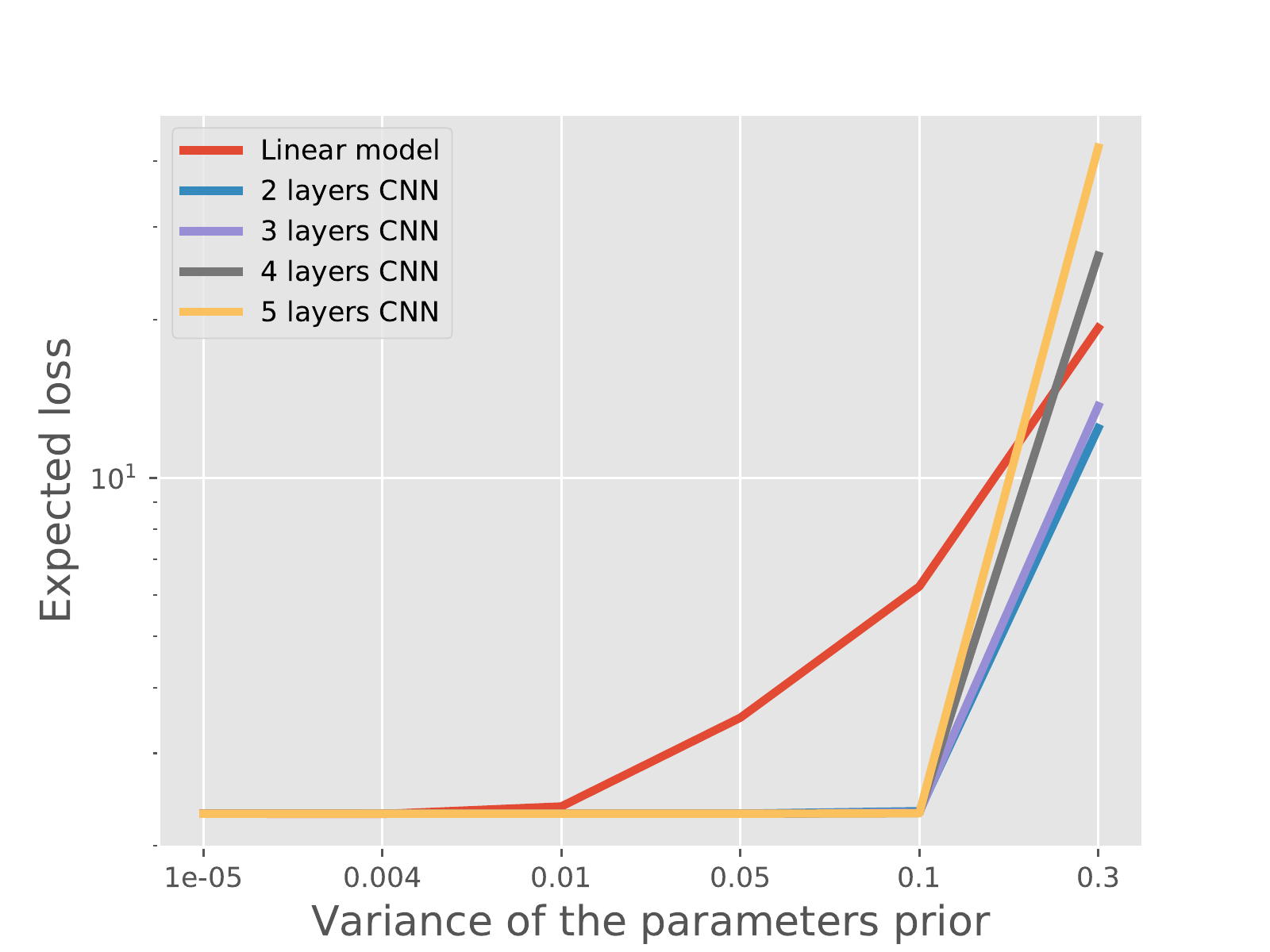}}
  \caption{Verifying the assumption in Corollary \ref{cor:nonlinear} that the loss function is on-average bounded, i.e., $\E_{(x,y) \sim D} \ell(w,x,y) \le b$. We tested this assumption on MNIST (left), Fashion-MNIST (middle), using MLPs and CIFAR10 (right) using CNNs. The loss is on-average bounded, while being dependent on the variance of the prior. For variance up to $0.1$, the loss is on-average bound is about $1$.}
     	\label{fig:exploss}
\end{figure*}

\begin{figure*}[t]
  \centering
  \subfloat{\includegraphics[width=0.3\textwidth, height=5cm, keepaspectratio]{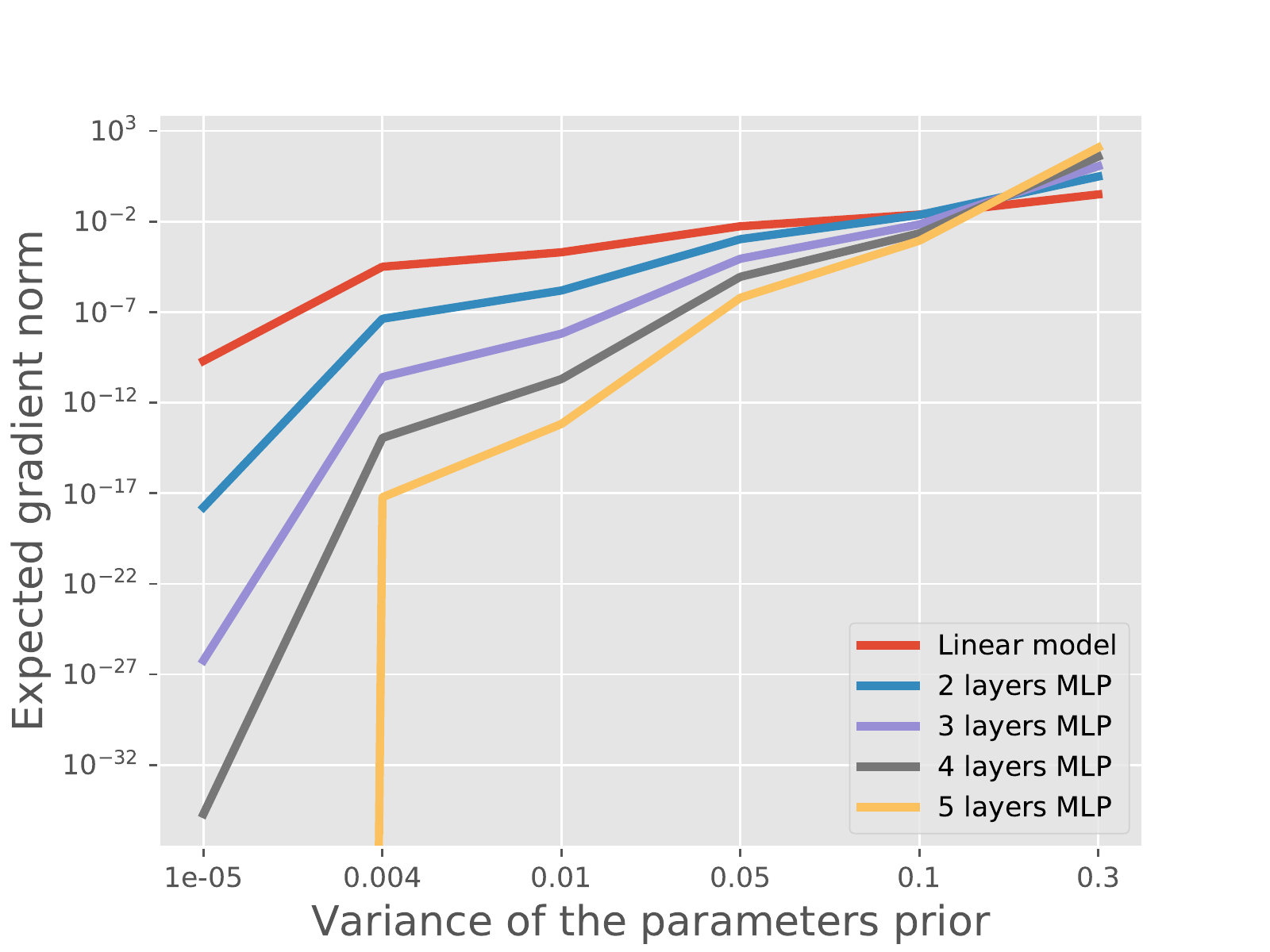}}
  \subfloat{\includegraphics[width=0.3\textwidth, height=5cm, keepaspectratio]{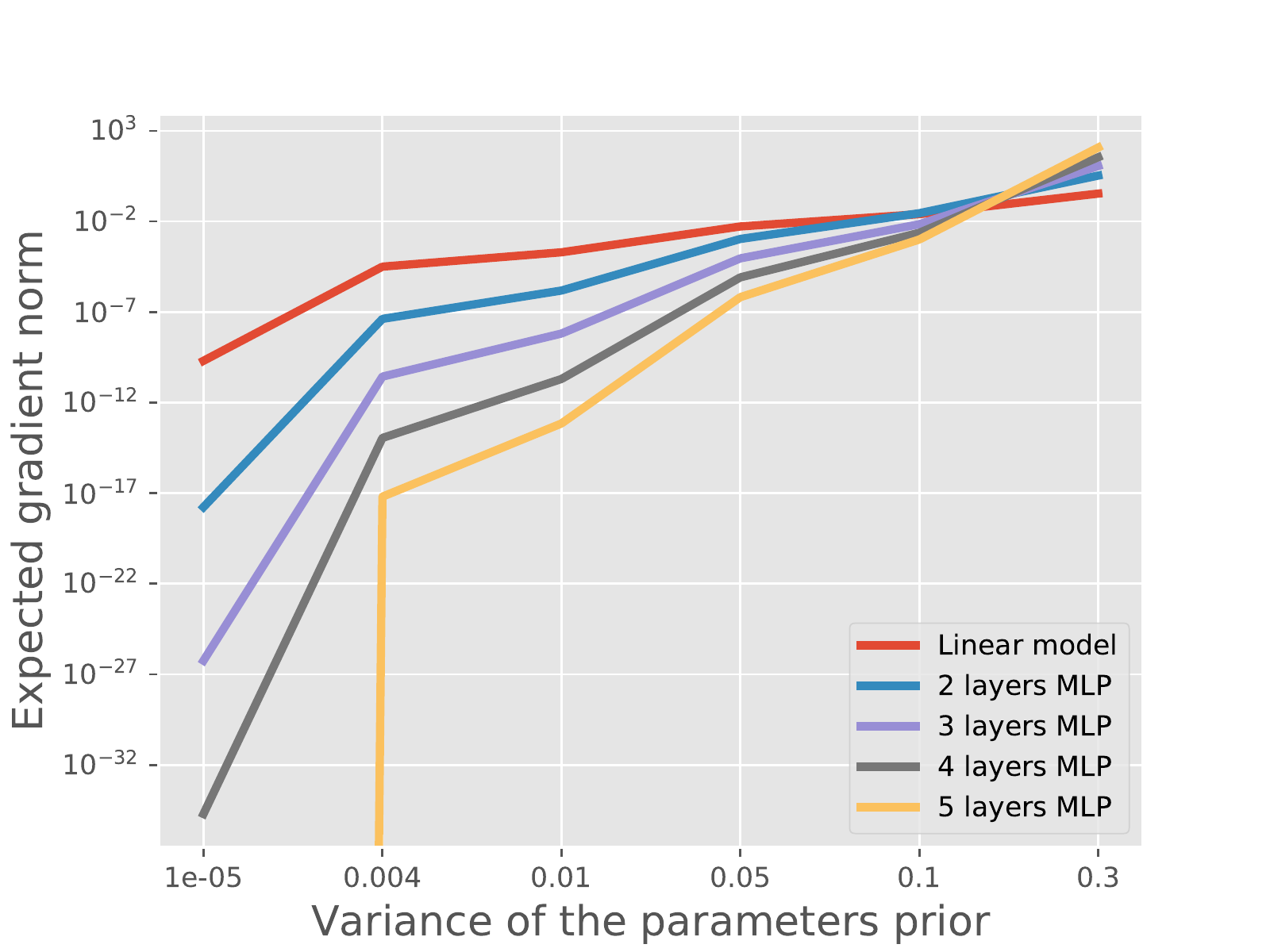}}
  \subfloat{\includegraphics[width=0.3\textwidth, height=5cm, keepaspectratio]{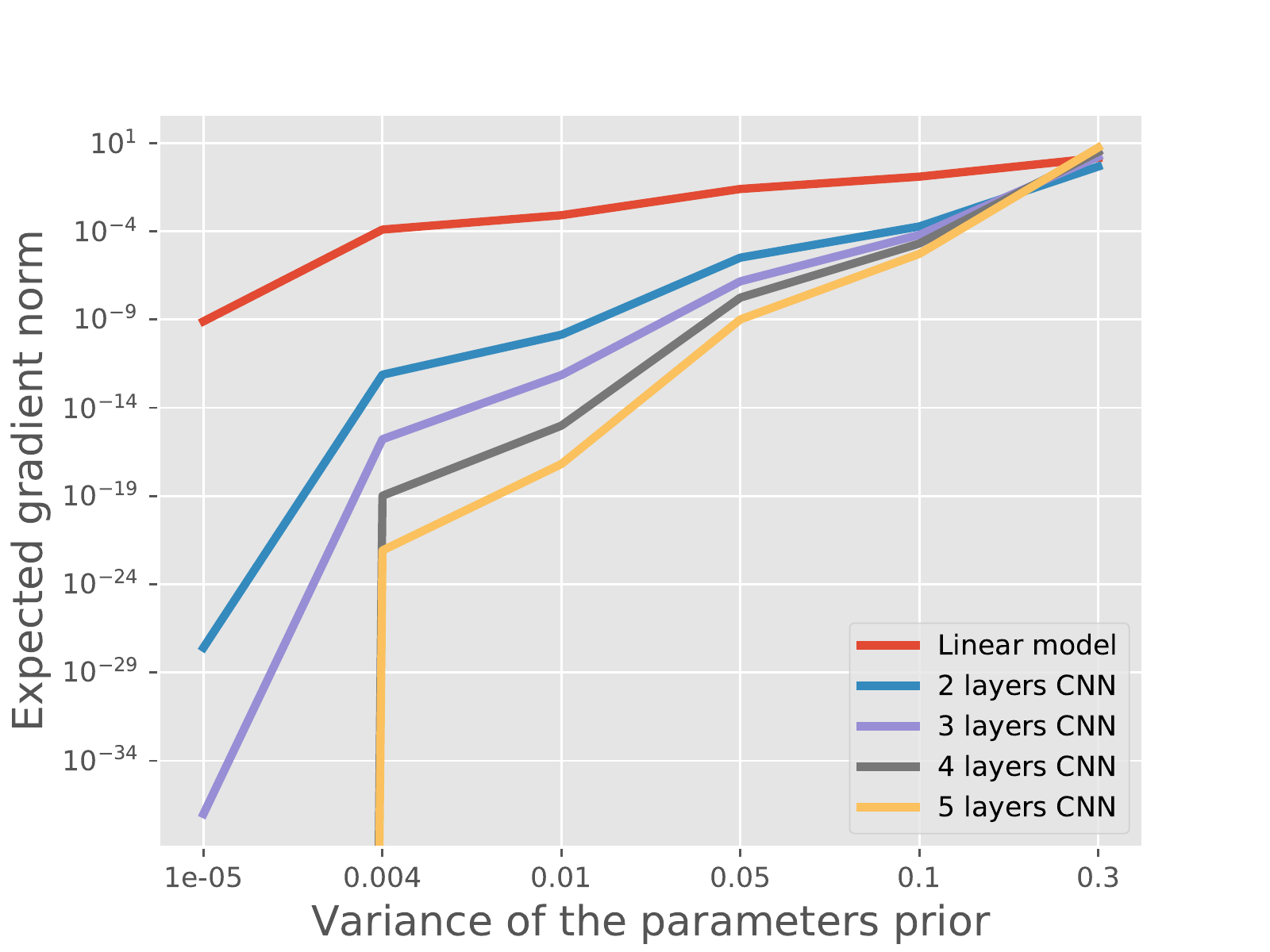}}
  \caption{Estimating $\E_{(x, y) \sim D}\big [ \| \nabla_{x} \ell(w, x, y) \| ^2 \big]$ as a function of the different variance levels for the prior distribution $p$. Results are reported for MNIST (left), Fashion-MNIST (middle), using MLPs and CIFAR10 (right) using CNNs. The linear model has the largest expected gradient-norm, since the Lipschitz condition considers the worst-case gradient-norm, see Corollary \ref{cor:logistic}. The gradient-norm gets smaller as a function of the depth, due to the vanishing gradient property. As a result, deeper nets can have faster convergence rate, i.e., use larger values of $\lambda$ can be used in the generalization bound.}
  	\label{fig:expgrad}
\end{figure*}

\begin{figure*}[!t]
  \centering
  \subfloat{\includegraphics[width=0.25\textwidth, height=4.5cm, keepaspectratio]{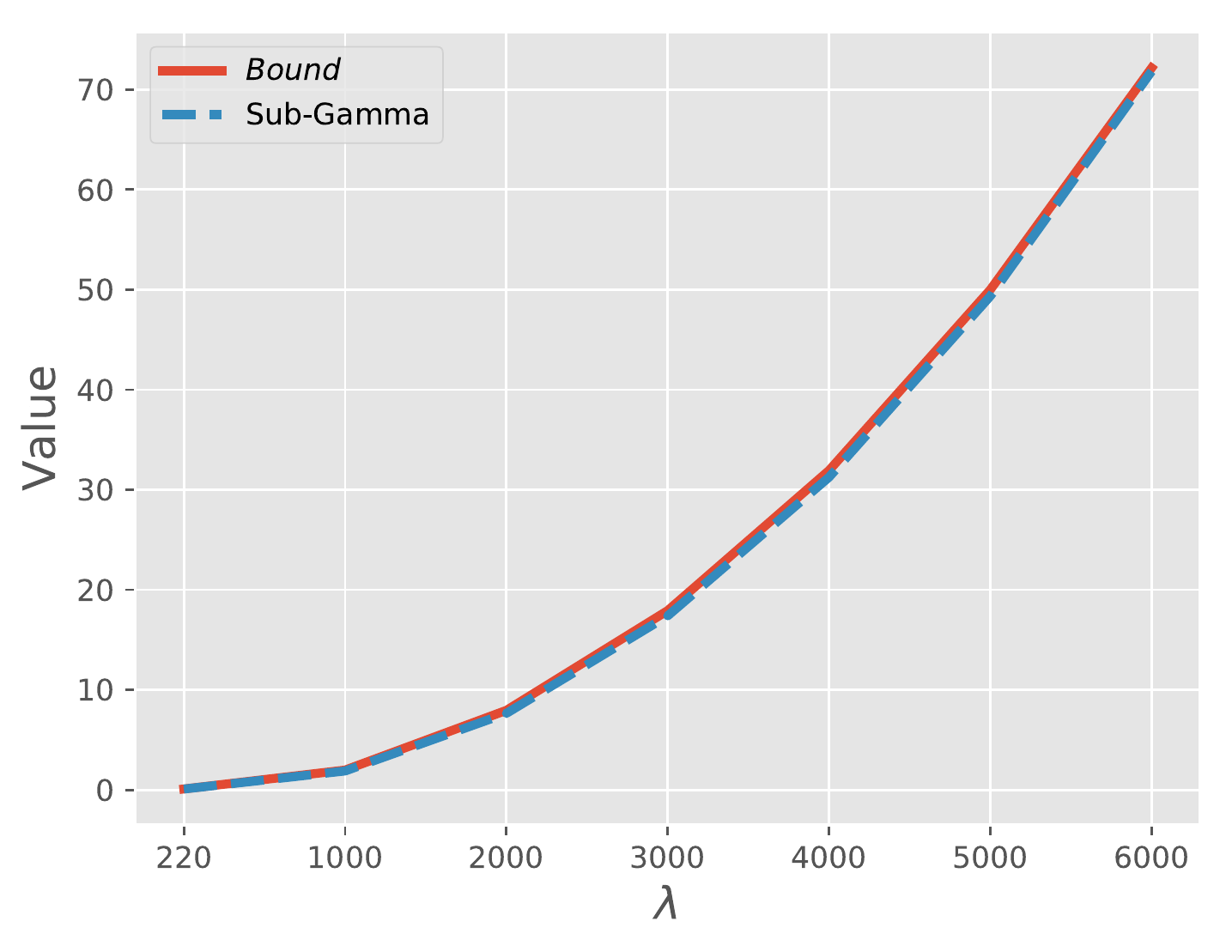}}
  \subfloat{\includegraphics[width=0.25\textwidth, height=4.5cm, keepaspectratio]{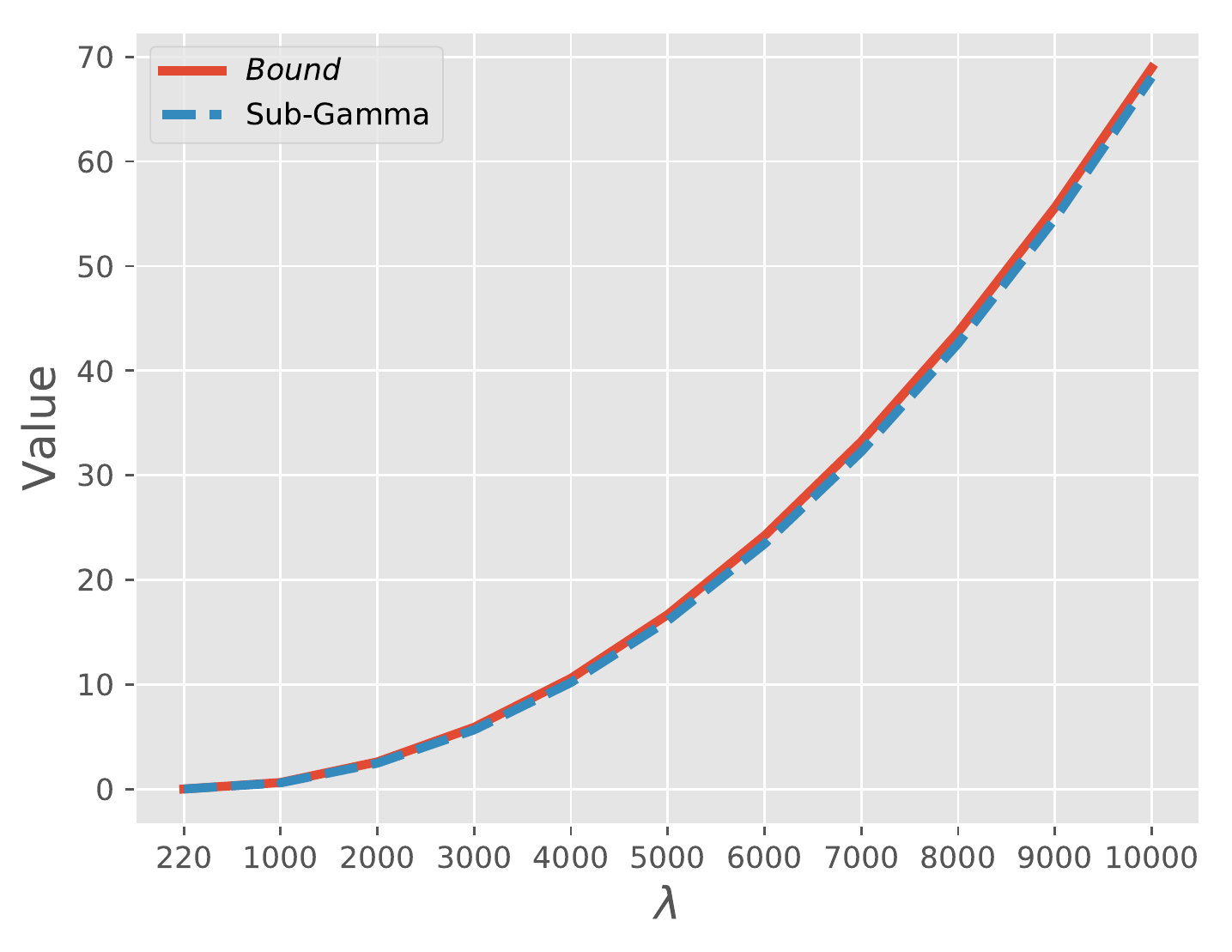}}
  \subfloat{\includegraphics[width=0.25\textwidth, height=4.5cm, keepaspectratio]{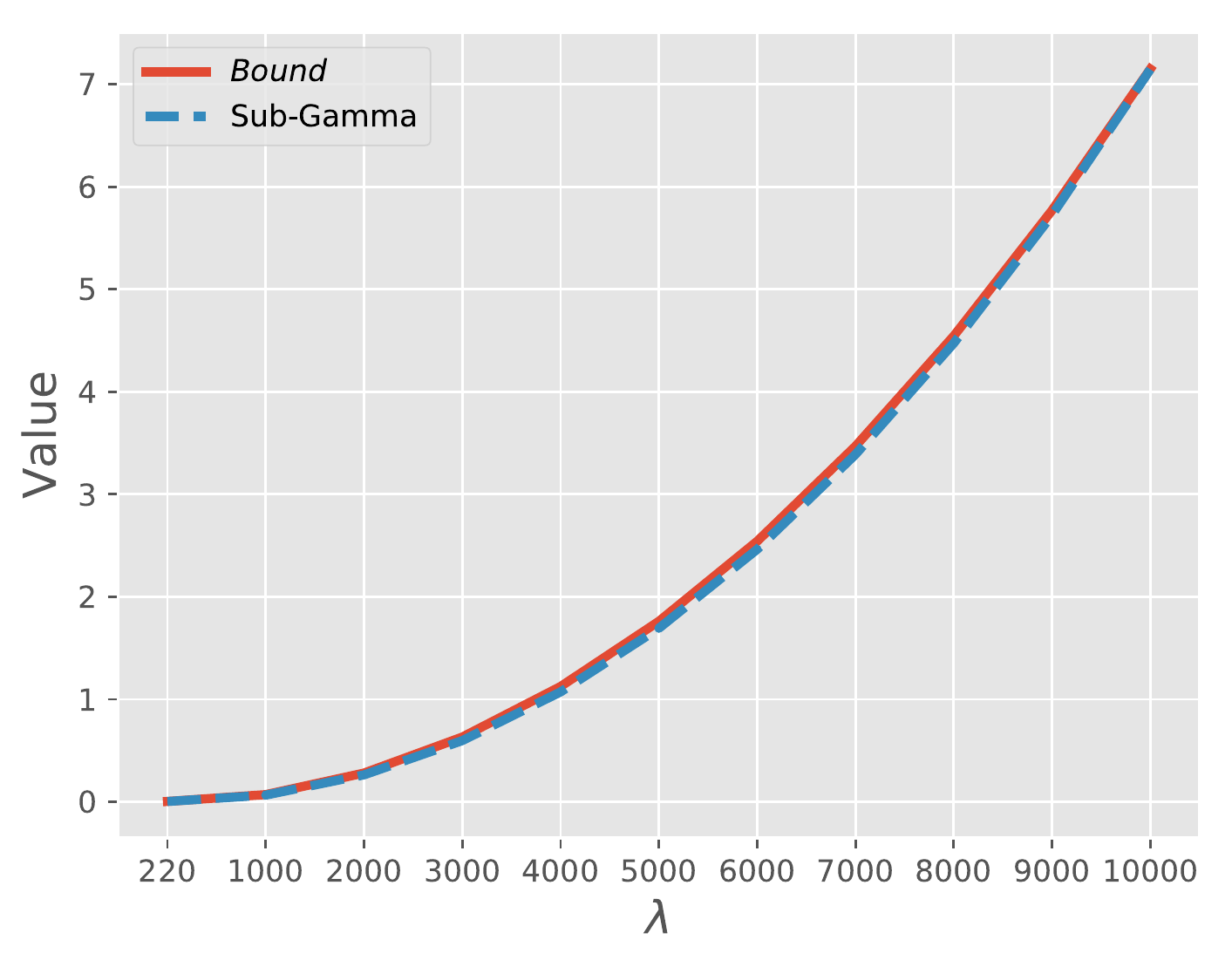}}
  \subfloat{\includegraphics[width=0.25\textwidth, height=4.5cm, keepaspectratio]{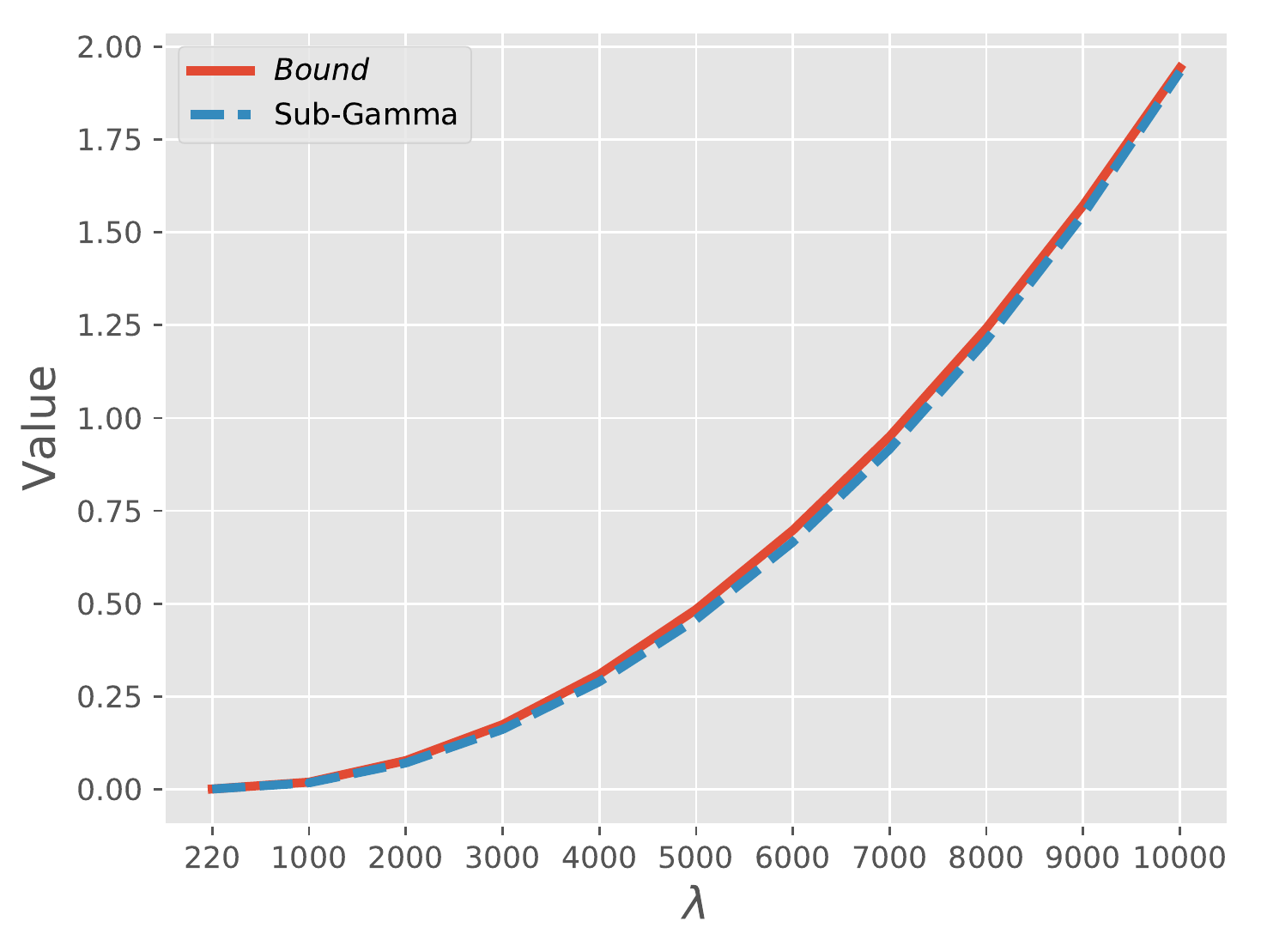}}
  \caption{The proposed bound as a function of $\lambda$ for MLP models using two, three, four, and five layers , depicted from left to right, (the left figure is for two layers, the second left is for three layers and so on). Notice, this suggests that the random variables $L_D(w) - L_S(w)$ for all presented MLP models are sub-gamma. We obtain the results for all models using the MNIST dataset. The parameter for the sub-gamma fit is $c=1\mathrm{e}{-5}$ for all models using different scaling factors.}
   \label{fig:gamma}
\end{figure*}

In this section we perform an experimental evaluation of the our PAC-Bayesian bounds, both for linear models and non-linear models. We begin by verifying our assumptions: (i) The PAC-Bayesian bound in Theorem \ref{thm:alquir} cannot be computed for large values of $\lambda$. (ii) Although the NLL loss is unbounded, it is on-average bounded. Next, we study the behavior of the complexity term $C(\lambda,p)$ for different architectures, both for linear models and deep nets. We show that the random variable $L_D(w) - L_S(w)$ is sub-gamma, namely that $C(\lambda,p) \le \frac{\lambda^2 v}{2(1-\lambda c)}$  for every $\lambda$ such that $0 < \lambda < 1/c$. Importantly, we show that $1/c$ which relates to the rate of convergence of the bound is determined by the architecture of the deep net. Lastly, we demonstrate the importance of $C(\lambda,p)$ in the learning process, balancing the three terms of (i) the empirical risk; (ii) the KL divergence; and (iii) the complexity term. 

\paragraph{Implementation details.} We use multilayer perceptrons (MLPs) for the MNIST and Fashion-MNIST dataset~\cite{xiao2017}. We use convolutional neural networks (CNNs) the CIFAR-10 dataset. In all models we use the ReLU activation function. We optimize the NLL loss function using SGD with a learning rate of 0.01 and a momentum value of 0.9 in all settings for 50 epochs. We use mini-batches of size 128 and did not use any learning rate scheduling. For the ResNet experiments we optimize an 18-layers ResNet model on the CIFAR-10 dataset, using Adam optimizer with learning rate of 0.001 for 150 epochs where we halve the learning rate every 50 epochs using batch-size of 128.

\subsection{Verify Assumptions}

We start by empirically demonstrating the numerical instability of computing the complexity term $C(\lambda,p)$ in Eq.~\ref{eq:alquier}. This numerical instability occurs due to the exponentiation of the random variables, namely $e^{\lambda(L_D(w) - L_S(w))}$, which quickly goes to infinity as $\lambda$ grows. We estimated Eq.~\ref{eq:alquier} over MNIST, for $5$ MLPs of depth $d=1,...,5$, where MLP of depth $1$ is a linear model. For a fair comparison we changed layers' width to reach roughly the same number of parameters in each model (except for the linear case). For these architectures evaluated $C(\lambda,p)$ for different values of $\lambda$, using variance of 0.1 over the prior distribution. The results are depicted in Figure~\ref{fig:alquier}. One can see that we are able to compute the complexity term $C(\lambda,p)$ for $\lambda \le 25$ and even to smaller $\lambda$ for the linear model ($\lambda \le 5$). Standard bounds for MNIST require $\lambda$ to be in the interval $\sqrt{m} < \lambda < m$, where $m$ is the training sample size ($m=60,000$ for MNIST). We observe that while computing $C(\lambda,p)$ the term $e^{\lambda L_D(w)}$ goes to infinity while $e^{-\lambda L_S(w)}$ goes to zero, and they are not able to balance each other. In our derivation this is solved by looking at the gradient to emphasize their change, see Eq.~\eqref{eq:insight}. 

\begin{figure}[!t]
  \centering
  \subfloat{\includegraphics[width=0.5\textwidth, height=4.5cm, keepaspectratio]{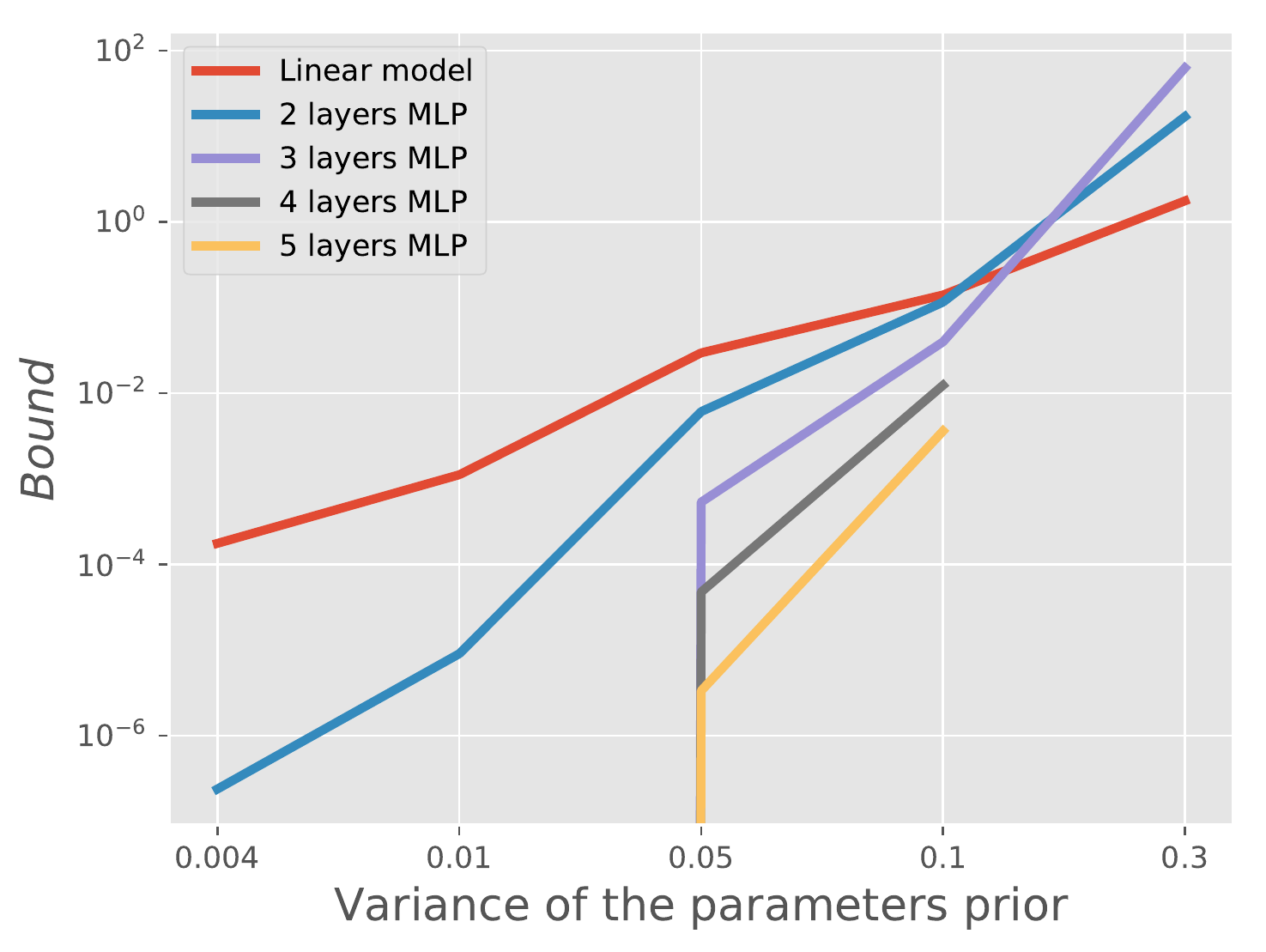}}
    \hfill
    \subfloat{\includegraphics[width=0.5\textwidth, height=4.5cm, keepaspectratio]{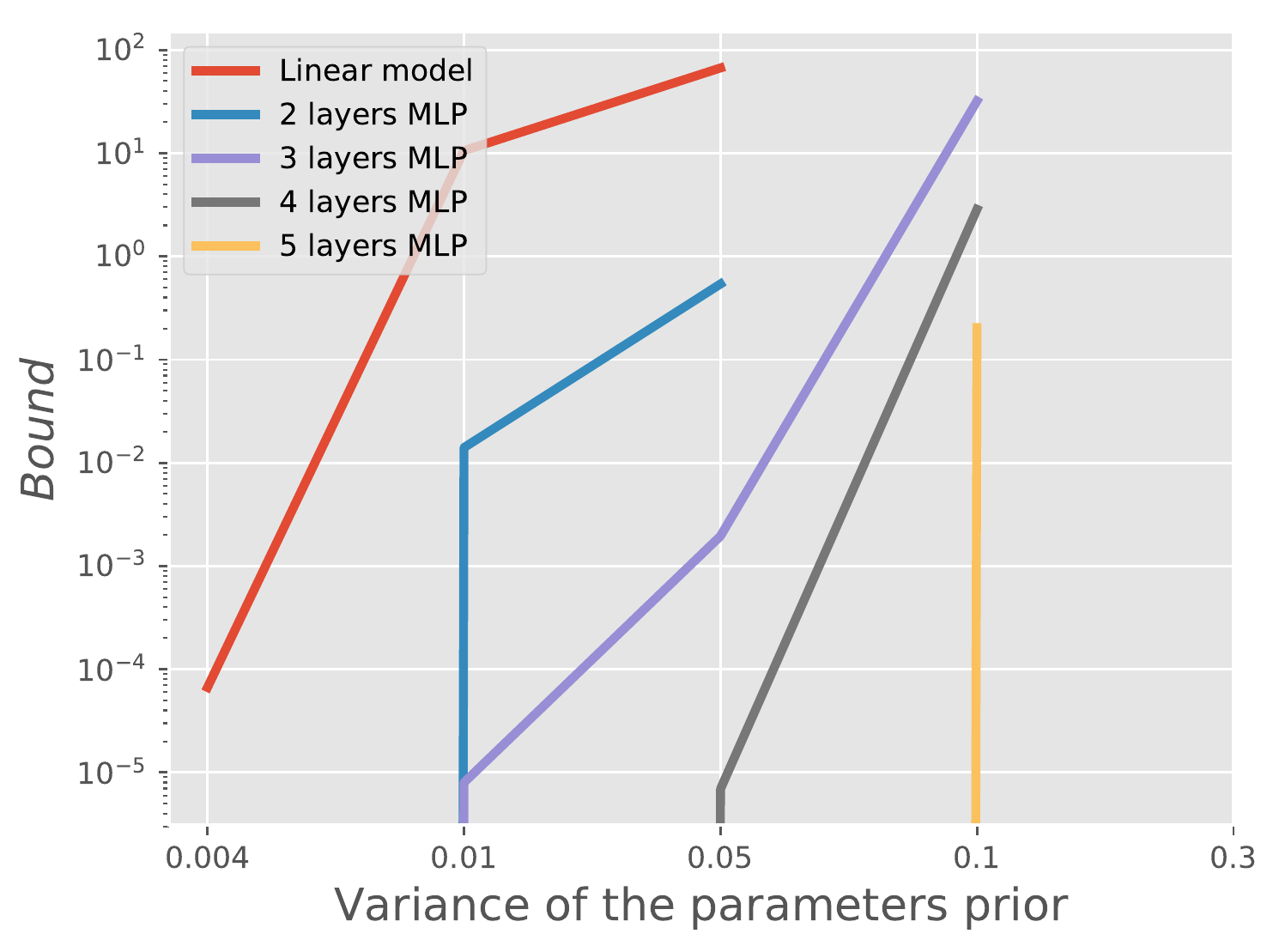}}
  \caption{Our bound on $C(\lambda,p)$ as a function of variance levels for the prior distribution over the models parameters, for $\lambda = \sqrt{m}$ (top) and $\lambda = m$ (bottom). One can see that the bound gets larger when the variance of the prior increases. This plot conforms the practice of Bayesian deep nets that set the variance of the prior up to $0.1$. Notice, the five layers for the $\lambda = m$, it get value only for variance of $0.1$, below that value the bound is zero, and above that value the bound explodes.}
   	\label{fig:mgf_bound}
\end{figure}

In Corollary \ref{cor:nonlinear} we assume that although the loss is unbounded, it is on-average bounded, meaning $\E_{(x,y) \sim D} \ell(w,x,y) \le b$. We tested this assumption on MNIST, Fashion-MNIST using MLPs of depth $d=1,...,5$ and CIFAR10 using CNNs of depth $d=2,...,5$, where for in the CNN models $d$ is the number of convolutional layers, and we include a max pool layer after each convolutional layer. In all CNN models we include an additional output layer to be a fully connected one. To span the possible weights we sampled them from a normal prior distribution with different variance. The results appear in Figure~\ref{fig:exploss}. We observe that the loss is on-average bounded by $10$, while being dependent on the variance of the prior. Moreover, for variance up to $0.1$, the on-average loss bound is about $1$ and its affect on the complexity term $C(\lambda,p)$ is minimal. Notice, although the expected loss is increasing for high variance levels, these are not being used for initialize deep nets, considering common initialization techniques.

\subsection{Complexity of Neural Nets}

 \begin{table*}[t]
 \caption{We optimize MLP models of different depth levels, where one corresponds to the linear model, two corresponds to two layers and so on. We report the avg. test loss, avg. train loss,  bound on $C(\lambda,p)$, and the KL value for the MNIST dataset.}
 \begin{center}
 \begin{tabular}{l|l|ccccccc}
 \hline
 & \bf Prior Variance &\bf 0.0004 & \bf 0.01 &  \bf 0.05 & \bf 0.1 &  \bf 0.3 & \bf 0.5 & \bf 0.7\\
 \hline
 \multirow{5}{*}{\rotatebox[origin=c]{90}{One}}
 & \bf Test Loss 					  & 0.753  		& 0.442  & 0.277    & 0.271  & 0.324    & 0.408 	& 0.677\\
 & \bf Train Loss 					  & 0.731  	    & 0.424  & 0.273    & 0.253  & 0.292    & 0.373 	& 0.697\\
 & \bf Bound on $C(m,p)$        & 10.57  		& 66.92  & inf 		& inf 	 & inf 		& inf 		& inf\\
 & \bf Bound on $C(\sqrt{m},p)$ & 0.0002  		& 0.001  & 0.029    & 0.141  & 1.76     & 6.33  	&  15.18\\
& \bf KL 							  & 20027     	& 10776  & 2561     & 1478	 & 3886	    & 7447 		& 8995\\
   \hline 
 \multirow{5}{*}{\rotatebox[origin=c]{90}{Two}}
 & \bf Test Loss 					  & 1.743  		& 0.549  & 0.104   & 0.066  & 0.127   & 0.236 	& 0.602\\
 & \bf Train Loss 					  & 1.752  	    & 0.569  & 0.095   & 0.038  & 0.056   & 0.155 	& 0.425\\
 & \bf Bound on $C(m,p)$        & 0.0 			& 0.014 & 0.540	& inf 	 & inf 		& inf 		& inf\\
 & \bf Bound on $C(\sqrt{m},p)$ & 0.0  		& 0.0	  & 0.006   & 0.115  & 16.95   & inf 		& inf\\
& \bf KL 							  & 25848     	& 30469   &  7369   & 7834	 & 96976    & 166638 	& 244255\\
 \hline
  \multirow{5}{*}{\rotatebox[origin=c]{90}{Three}}
 & \bf Test Loss 					  & 2.3  		& 2.3  	& 0.091  	& 0.062 & 0.136   & 0.294 	& 1.319\\
 & \bf Train Loss 					  & 2.3  	    & 2.3  	& 0.078  	& 0.027 & 0.067   & 0.268 	& 1.173\\
 & \bf Bound on $C(m,p)$        & 0.0  		& 0.0 		& 0.002  	& 31.99  & inf 		& inf 		& inf\\
 & \bf Bound on $C(\sqrt{m},p)$ & 0.0  		& 0.001  	& 0.001  	& 0.041 & 62.73   	& inf 		& inf\\
& \bf KL 							  & nan     		& 10776   	& 9480   	& 8215	 & 95984	& 175134 	& 226237\\
 \hline
  \multirow{5}{*}{\rotatebox[origin=c]{90}{Four}}
 & \bf Test Loss 					  & 2.3  		& 2.3  		& 0.083  	& 0.067 & 0.132   & inf 	& inf\\
 & \bf Train Loss 					  & 2.3  	    & 2.3	  	& 0.064  	& 0.022 & 0.081   & inf 	& inf\\
 & \bf Bound on $C(m,p)$        & 0.0  		& 0.0 		& 0.0  		& 2.855  & inf 		& inf 	& inf\\
 & \bf Bound on $C(\sqrt{m},p)$ & 0.0  		& 0.0  		& 0.0  		& 0.012 & inf   	& inf 	& inf\\
& \bf KL 							  & nan     		& nan   		& 10849   	& 8239	 & 113943	& nan 	& nan\\
 \hline
  \multirow{5}{*}{\rotatebox[origin=c]{90}{Five}}
  & \bf Test Loss 					  & 2.3  		& 2.3  	& 0.087  	& 0.066 	& 0.133   	& inf 	& inf\\
 & \bf Train Loss 					  & 2.3  	    & 2.3  	& 0.055  	& 0.019 	& 0.101   	& inf 	& inf\\
 & \bf Bound on $C(m,p)$        & 0.0  		& 0.0 		& 0.0  		& 0.204  	& inf 		& inf 	& inf\\
 & \bf Bound on $C(\sqrt{m},p)$ & 0.0  		& 0.0  		& 0.0  		& 0.004 	& inf   	& inf 	& inf\\
& \bf KL 							  & nan     		& nan   		& 11800   	& 9090	 	& 140305	& nan 	& nan\\
 \hline 
 \end{tabular}
 \label{tab:optimize}
 \end{center}
 \end{table*}
 
Next we turn to estimate our bounds of $C(\lambda,p)$, both for the linear models and nonlinear models, corresponding to Corollary \ref{cor:logistic} and Corollary \ref{cor:nonlinear}. We use the same architectures as mentioned above. The bound on $C(\lambda,p)$ is controlled by the expected gradient-norm $\E_{(x, y) \sim D}\big [ \| \nabla_{x} \ell(w, x, y) \| ^2 \big]$. Figure~\ref{fig:expgrad} presents the expected gradient-norm as a function of different variance levels for the prior distribution $p$ over the models parameters. For the linear model we used the bound in Corollary \ref{cor:logistic}. One can see that the linear model has the largest expected gradient-norm, since the Lipschitz condition considers the worst-case gradient-norm. One also can see that the deeper the network, the smaller its gradient-norm. This is attributed to the gradient vanishing property. As a result, deeper nets can have faster convergence rate, i.e., use larger values of $\lambda$ in the generalization bound, since the vanishing gradients creates a contractivity property that stabilize the loss function, i.e., reduces its variability. However, this comes at the expanse of the expressivity of the deep net, since vanishing gradients cannot fit the training data in the learning phase. This is demonstrated in the next experiment.    

Figure~\ref{fig:mgf_bound} presents the bound on $C(\lambda,p)$ as a function of the variance levels for the prior distribution over the models parameters. One can see that the bound gets larger when the variance of the prior increases. Another thing to note is that the bound is often explodes when the variance is larger than $0.3$. This conforms with Corollary \ref{cor:logistic}, which $C(\lambda,p)$ is unbounded for variance larger than $0.5$. This plot conforms the practice of Bayesian deep nets that set the variance of the prior up to $0.1$. 

\subsection{Sub-Gamma Approximation}
\label{sec:subgamma}
In Section~\ref{sec:linear} we proved that the proposed bound is sub-gamma for the linear case. Unfortunately, such proof can not be directly applied to the non-linear case. Hence, we empirically demonstrate that the proposed bound over $C(\lambda, p)$ is indeed sub-gamma for various model architectures. For that we used the same models architectures as before using the MNIST dataset. Results are depicted in~Figure~\ref{fig:gamma}. Notice, similar to the ResNet model, the proposed bound is sub-gamma in all explored settings using $c=1e-5$ with different scaling factors.

\subsection{Optimization}

Lastly, in order to better understand the balance between all components composed the proposed generalization bound we optimize all five MLP models presented above using the MNIST dataset, and computed the average training loss, average test loss, KL divergence, and the bound on $C(\lambda,p)$, bound using $\lambda=m$ and $\lambda=\sqrt{m}$. We repeat this optimization process for various variance levels over the prior distribution over the model parameters. Results for the MNIST dataset are summerized in Table~\ref{tab:optimize}, more experimental results can be found in Section~\ref{sec:appendix_res} in the Appendix. 

Results suggests that using variance levels of [0.05, 0.1] produce the overall best performance across all depth levels. This findings is consistent with our previous results which suggest that below this value the Bound goes to zero, hence make a good generalization on the expense of model performance. However, larger variance levels may cause the bound to explode and as a results makes the optimization problem harder.

\begin{figure*}[!t]
  \centering
  \subfloat{\includegraphics[width=0.3\textwidth, height=4.5cm, keepaspectratio]{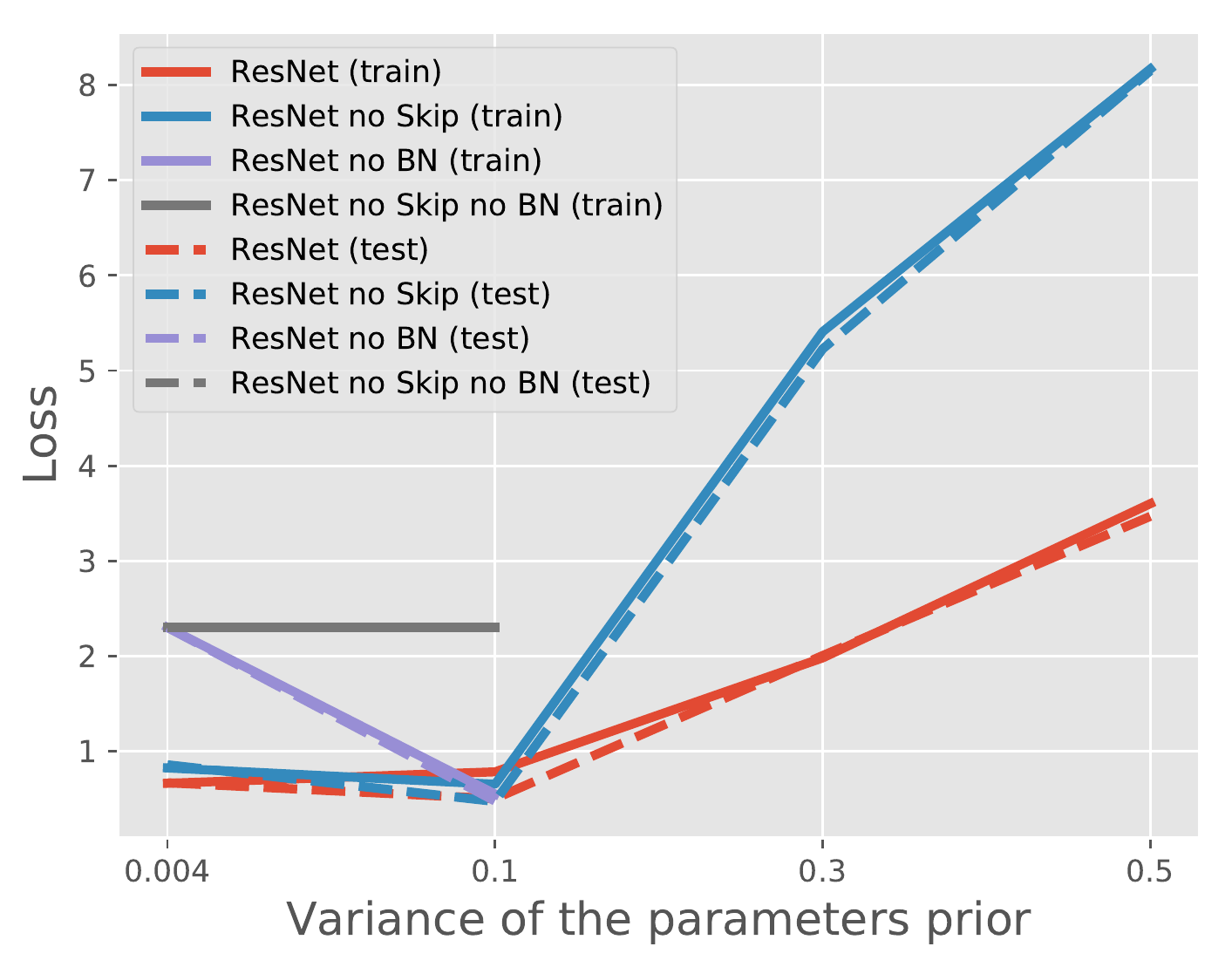}}
  \subfloat{\includegraphics[width=0.3\textwidth, height=4.5cm, keepaspectratio]{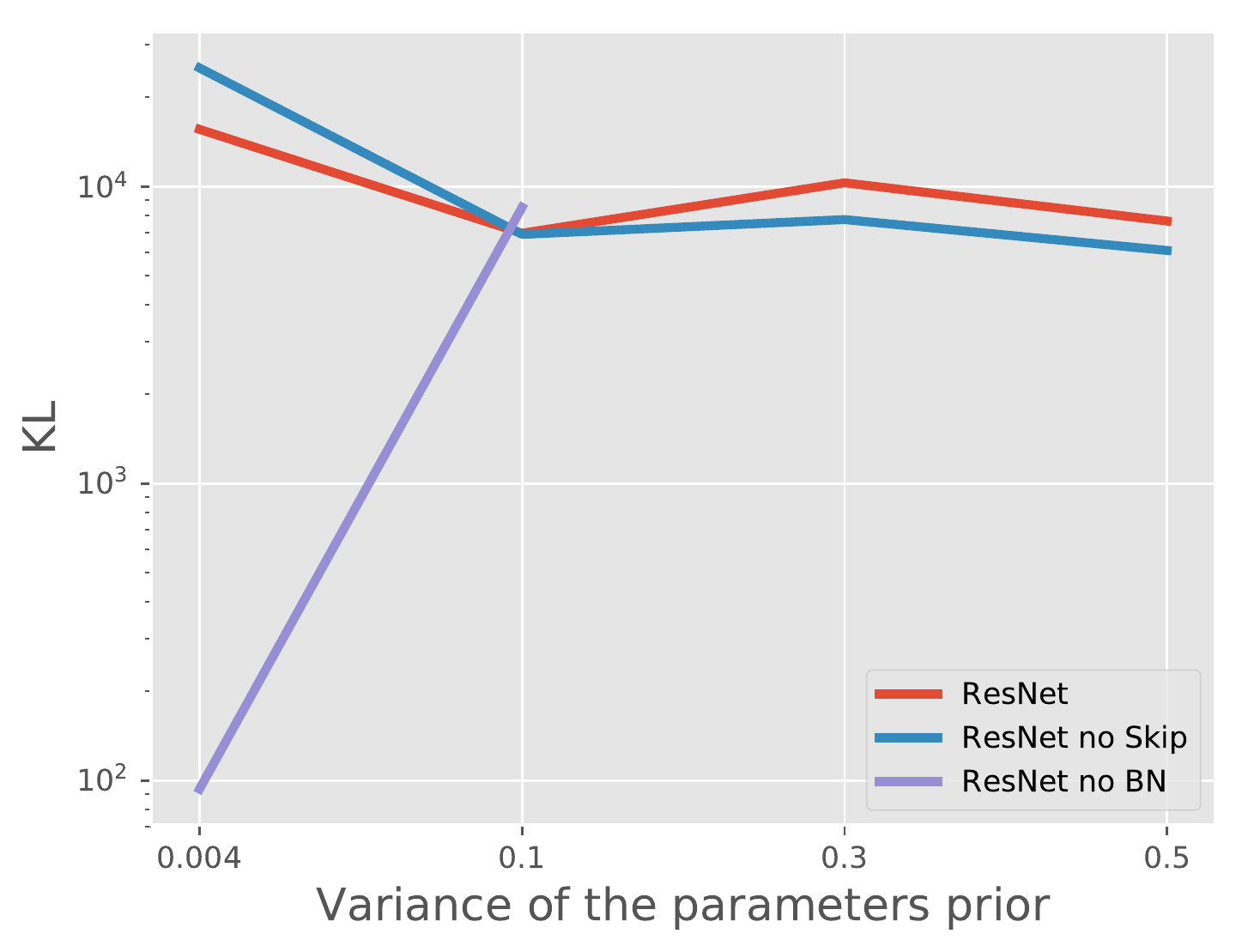}}
  \subfloat{\includegraphics[width=0.3\textwidth, height=4.5cm, keepaspectratio]{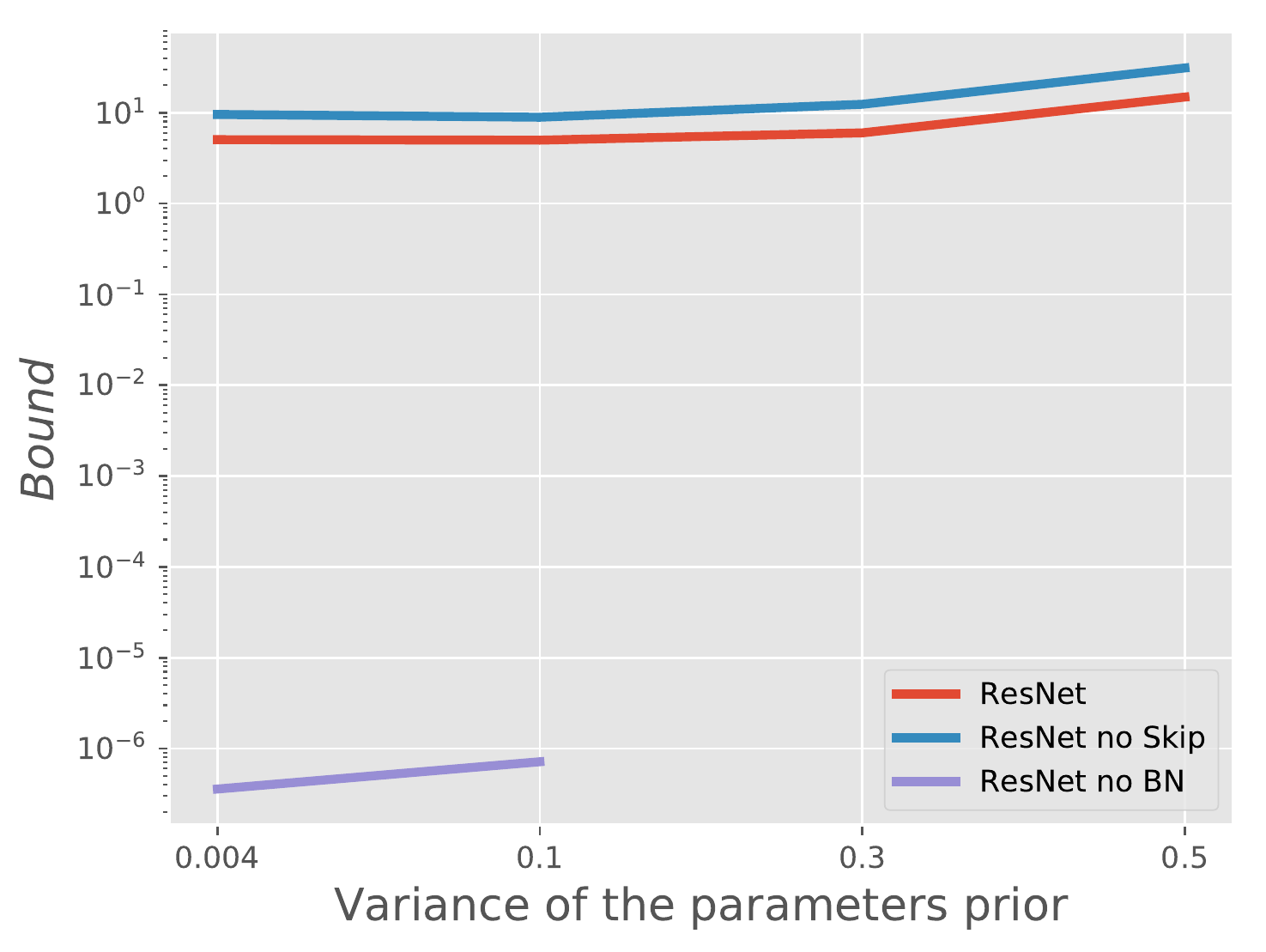}}
  \caption{ResNet variations results for CIFAR10 dataset. In the left subplot we report the average training loss and average test loss (dashed lines). In the middle sub figure we present the KL values, and in the right hand sub figure we report the bound on $C(\lambda,p)$. All results are reported using different variance levels of the prior distribution over the model parameters.}
   \label{fig:resnet}
\end{figure*}

Lastly, we analyzed the commonly used ResNet model~\cite{he2016deep}. For that, we trained four different versions of the ResNet18 model: (i) standard model (ResNet); (ii) model with no skip connections (ResNetNoSkip); (iii) model with no batch normalizations (ResNetNoBN); and (iv) model without both skip connections  and batch normalization layers (ResNetNoSkipNoBN). We optimize all models using CIFAR-10 dataset. Figure~\ref{fig:resnet} visualizes the results. Consistently with previous findings, variance levels of 0.1 gets the best performance overall, both in terms of model test loss and the generalization.

Notice, ResNet and ResNetNoSkip achieve comparable performance in all measures. Additionally, when considering variance levels of 0.1 for the prior distribution, removing the batch normalization layers and including the skip-connections also gets comparable performance to ResNet and ResNetNoSkip. Similarly to~\citet{zhang2019fixup}, this findings suggest that even without batch normalization layers models can converge using exact initialization. On the other hand while removing both batch normalization the and skip connections, models either explores immediately or suffer greatly from gradient vanishing. These results are consistent with previous findings in which batch normazliaton greatly improves optimization~\cite{santurkar2018does}.

\section{Discussion and Future Work}

We present a new PAC-Bayesian generalization bound for deep Bayesian neural networks for unbounded loss functions with unbounded gradient-norm. The proof relies on bounding the log-partition function using the expected squared norm of the gradients with respect to the input. We prove that the proposed bound is sub-gamma for any linear model with a Lipschitz loss function and we verify it empirically for the non-linear case. Experimental validation shows that the resulting bound provides insights for better model optimization, prior distribution search and model initialization. 

\bibliographystyle{icml2020}
\bibliography{iclr2020}

\newpage
\onecolumn
\appendix
\section{Appendix}
\subsection{Extended results}
\label{sec:appendix_res}
We provide an additional results for the optimization experiments. We optimized MLP models at different depth levels using Fashion-MNIST dataset.

 \begin{table}[h]
 \caption{We optimize models of different depth levels, where one corresponds to the linear model, two corresponds to two layers and so on. We report the avg. test loss, avg. train loss,  bound on $C(\lambda,p)$, and the KL value.}
 \begin{center}
 \begin{tabular}{l|l|ccccccc}
  \hline
 & \bf Prior Variance &\bf 0.0004 & \bf 0.01 &  \bf 0.05 & \bf 0.1 &  \bf 0.3 & \bf 0.5 & \bf 0.7\\
 \hline 
  \multirow{5}{*}{\rotatebox[origin=c]{90}{One}}
 & \bf Test Loss 					  & 0.79  	& 0.57   & 0.45   & 0.46  & 1.04  & 1.59  & 10.43\\
 & \bf Train Loss 					  & 0.77  	& 0.55   & 0.41   & 0.39  & 0.86  & 1.31  & 9.97\\
 & \bf MGF bound ($\lambda=m$)        & 10.57  	& 67.31  & inf    & inf   & inf   & inf   & inf\\
 & \bf MGF bound ($\lambda=\sqrt{m}$) & 0.0001  	& 0.0011 & 0.0296 & 0.143 & 1.812 & 6.386 & 14.74\\
 & \bf KL							  & 14255  	& 7829   & 2430   & 1648  & 5379  & 8440  & 35299\\
  \hline 
 \multirow{5}{*}{\rotatebox[origin=c]{90}{Two}}
 & \bf Test Loss 					  & 1.46  	& 0.71   & 0.35   & 0.31  & 0.42   & 0.6     & 10.43\\
 & \bf Train Loss 					  & 1.46  	& 0.69   & 0.29   & 0.2   & 0.3    & 0.49    & 9.97\\
 & \bf MGF bound ($\lambda=m$)        & 0.014  	& 0.53   & inf    & inf   & inf    & inf     & inf\\
 & \bf MGF bound ($\lambda=\sqrt{m}$) & 0.0  	& 0.0    & 0.005  & 0.12  & 16.86  & inf     & inf\\
 & \bf KL							  & 28222  	& 22053  & 6983   & 9180  & 98648  & 168266  & 209935\\
 \hline
  \multirow{5}{*}{\rotatebox[origin=c]{90}{Three}}
 & \bf Test Loss 					  & 2.3  	& 0.92   & 0.34   & 0.31  & 0.42   & 0.83   & 2.39\\
 & \bf Train Loss 					  & 2.3  	& 0.91   & 0.28   & 0.19  & 0.35   & 0.79   & 1.76\\
 & \bf MGF bound ($\lambda=m$)        & 0.0  	& 0.0    & 0.002  & 32.16 & inf    & inf    & inf\\
 & \bf MGF bound ($\lambda=\sqrt{m}$) & 0.0  	& 0.0    & 0.0    & 0.04  & 64.48  & inf    & inf\\
 & \bf KL							  & 0     	& 29826  & 9219   & 1648  & 547636 & 198760 & 374718\\
 \hline
  \multirow{5}{*}{\rotatebox[origin=c]{90}{Four}}
 & \bf Test Loss 					  & 2.3  	& 2.3    & 0.34    & 0.31  & 0.44   & inf    & inf\\
 & \bf Train Loss 					  & 2.3  	& 2.3    & 0.26    & 0.18  & 0.38   & inf    & inf\\
 & \bf MGF bound ($\lambda=m$)        & 0.0  	& 0.0    & 0.0     & 2.81  & inf    & inf    & inf\\
 & \bf MGF bound ($\lambda=\sqrt{m}$) & 0.0  	& 0.0    & 0.0     & 0.01  & inf    & inf    & inf\\
 & \bf KL							  & 0     	& 0      & 10859   & 10165 & 388842 & nan    & nan\\
 \hline
  \multirow{5}{*}{\rotatebox[origin=c]{90}{Five}}
 & \bf Test Loss 					  & 2.3  	& 2.3    & 0.34    & 0.31   & 1.3    & inf    & inf\\
 & \bf Train Loss 					  & 2.3  	& 2.3    & 0.25    & 0.18   & 1.28   & inf    & inf\\
 & \bf MGF bound ($\lambda=m$)        & 0.0  	& 0.0    & 0.0     & 0.19   & inf    & inf    & inf\\
 & \bf MGF bound ($\lambda=\sqrt{m}$) & 0.0  	& 0.0    & 0.0     & 0.003  & inf    & inf    & inf\\
 & \bf KL							  & 0     	& 0      & 11918   & 10986  & 353632 & nan    & nan\\
 \hline 
 \end{tabular}
 \label{tab:optimize_app}
 \end{center}
 \end{table}
\end{document}